\RequirePackage{fix-cm}
\documentclass[twocolumn]{svjour3}          
\smartqed  
\usepackage{graphicx}
\usepackage{amsmath,amssymb} 
\usepackage{lineno}
\usepackage{color}
\usepackage{algorithm}
\usepackage{algorithmic}

\journalname{Journal}

\newcommand{\assign}{:=}
\newcommand{\cdummy}{\cdot}
\newcommand{\nocomma}{}
\newcommand{\tmem}[1]{{\em #1\/}}

\newcommand{\tmop}[1]{\ensuremath{\operatorname{#1}}}
\newcommand{\tmstrong}[1]{\textbf{#1}}

\newcommand{\tmtextit}[1]{{\itshape{#1}}}

\newcommand{\mathbbm}{\mathbb}
\usepackage{amsmath,amssymb,graphicx}

\newcommand{\cont}{\nonumber}
\newcommand{\mcont}{\cdot\cont}
\newcommand{\qq}{\quad\quad}
\newcommand{\qqq}{\quad\quad\quad}
\newcommand{\qqqq}{\quad\quad\quad\quad}
\newcommand{\nn}{\nonumber}

\usepackage[USenglish]{babel}
\begin{document}

\title{Optimality Bounds for a Variational Relaxation of the Image Partitioning Problem
}


\author{Jan Lellmann \and Frank Lenzen \and Christoph Schn{\"o}rr}


\institute{J.~Lellmann
\at Image and Pattern Analysis Group \& HCI\\
  Dept.~of Mathematics and Computer Science, University of Heidelberg
  \at \emph{Current Address:} Dept.~of Applied Mathematics and Theoretical Physics, University of Cambridge, United Kingdom\\
\email{j.lellmann@damtp.cam.ac.uk}
\and
F.~Lenzen \and C.~Schn{\"o}rr
\at Image and Pattern Analysis Group \& HCI\\
  Dept.~of Mathematics and Computer Science\\
   University of Heidelberg\\
  \email{lenzen@iwr.uni-heidelberg.de, schnoerr@math.uni-heidelberg.de}
}

\date{Received: date / Accepted: date}

\maketitle

\begin{abstract}
  We consider a variational convex relaxation of a class of optimal partitioning and multiclass labeling problems,
  which has recently proven quite successful and can be seen as a continuous analogue of Linear Programming (LP)
  relaxation methods for
  finite-dimensional problems. While for the latter case several optimality bounds are known, to our knowledge
  no such bounds exist in the continuous setting. We provide such a bound by analyzing a probabilistic rounding
  method, showing that it is possible to obtain an integral solution of the original partitioning problem
  from a solution of the relaxed problem with an \emph{a priori} upper bound on the objective, ensuring the
  quality of the result from the viewpoint of optimization. The approach has a natural interpretation as an
  approximate, multiclass variant of the celebrated coarea formula.
\end{abstract}


\section{Introduction and Background}

\subsection{Convex Relaxations of Partitioning Problems}
In this work, we will be concerned with a class of {\tmem{variational}}
problems used in image processing and analysis for formulating multiclass
image partitioning problems, which are of the form
\begin{eqnarray}
  \inf_{u \in \mathcal{C}_{\mathcal{E}}} f (u) & \assign &
  \int_{\Omega} \langle u (x), s (x) \rangle dx + \int_{\Omega} d \Psi (D u)
  \hspace{0.25em},  \label{eq:combprob}\\
  \mathcal{C}_{\mathcal{E}} & \assign & \tmop{BV} (\Omega, \mathcal{E}) \\
  & = & \{u \in \tmop{BV} (\Omega)^l |u (x) \in \mathcal{E} \text{\tmop{ for} a.e. ~}
  x \in \Omega\}, \\
  \mathcal{E} & \assign & \{e^1, \ldots, e^l \} .
\end{eqnarray}
The {\tmem{labeling function}} $u : \Omega \rightarrow \mathbbm{R}^l$ assigns
to each point in the image domain $\Omega \assign (0, 1)^d$ a label $i \in
\mathcal{I} \assign \{1, \ldots, l\}$, which is represented by one of the
$l$-dimensional unit vectors $e^1,\ldots,e^l$. Since it is piecewise constant and therefore cannot be
assumed to be differentiable, the problem is formulated as a
{\tmem{free discontinuity problem}} in the space $\tmop{BV} (\Omega,
\mathcal{E})$ of functions of bounded variation; we refer to
{\cite{Ambrosio2000}} for a general overview.

The objective function $f$ consists of a data term and a regularizer. The data term is
given in terms of the $L^1$ function $s (x) = (s_1 (x), \ldots, s_l (x)) \in
\mathbbm{R}^l$, and assigns to the choice $u (x) = e^i$ the ``penalty'' $s_i
(x)$, in the sense that
\begin{eqnarray}
  \int_{\Omega} \langle u (x), s (x) \rangle d x & = & \sum_{i = 1}^l
  \int_{\Omega_i} s_i (x) d x, 
\end{eqnarray}
where $\Omega_i \assign u^{- 1} (\{e^i \}) = \{x \in \Omega |u (x) = e^i \}$
is the {\tmem{class region}} for label $i$, i.e.,~the set of points that are
assigned the $i$-th label. The data term generally depends on the input data
-- such as color values of a recorded image, depth measurements, or other
features -- and promotes a good fit of the minimizer to the input data. While
it is purely local, there are no further restrictions such as continuity,
convexity etc., therefore it covers many interesting applications such as
segmentation, multi-view reconstruction, stitching, and inpainting
{\cite{Paragios2006}}.

\subsection{Convex Regularizers}
The {\tmem{regularizer}} is defined by the positively homogeneous, continuous
and convex function $\Psi : \mathbbm{R}^{d \times l} \rightarrow
\mathbbm{R}_{\geqslant 0}$ acting on the distributional derivative $D u$ of
$u$, and incorporates additional prior knowledge about the ``typical''
appearance of the desired output. For piecewise constant $u$, it can be seen
that the definition in (\ref{eq:combprob}) amounts to a weighted
penalization of the {\tmem{discontinuities}} of $u$:
\begin{eqnarray}
  & & \int_{\Omega} d \Psi (D u) = \\ \label{eq:dpsiduju}
  & & \quad \int_{J_u} \Psi (\nu_u (x) (u^+ (x) - u^-
  (x))^{\top}) d\mathcal{H}^{d - 1} (x),\nonumber
\end{eqnarray}
where $J_u$ is the jump set of $u$, i.e.,~the set of points where $u$ has
well-defined right-hand and left-hand limits $u^+$ and $u^-$ and (in an
infinitesimal sense) jumps between the values $u^+ (x), u^- (x) \in
\mathbbm{R}^l$ across a hyperplane with normal $\nu_u (x) \in \mathbbm{R}^d$,
$\| \nu_u (x)\|_2 = 1$ (see {\cite{Ambrosio2000}} for the precise definitions).

A particular case is to set $\Psi = (1 / \sqrt{2}) \| \cdot \|_2$, i.e.,~the
scaled Frobenius norm. In this case $J (u)$ is just the (scaled) total
variation of $u$, and, since $u^+ (x)$ and $u^- (x)$ assume values in
$\mathcal{E}$ and cannot be equal on the jump set $J_u$, it holds that
\begin{eqnarray}
  J (u) & = & \frac{1}{\sqrt{2}} \int_{J_u} \|u^+ (x) - u^- (x)\|_2
  d\mathcal{H}^{d - 1} (x),  \label{eq:jusimple}\\
  & = & \mathcal{H}^{d - 1} (J_u) . 
\end{eqnarray}
Therefore, for $\Psi = (1 / \sqrt{2}) \| \cdot \|_2$ the regularizer just
amounts to penalizing the total length of the interfaces between class regions
as measured by the $(d - 1)$-dimensional Hausdorff measure $\mathcal{H}^{d -
1}$, which is known as {\tmem{uniform metric}} or {\tmem{Potts}} regularizer.

A general regularizer was proposed in {\cite{Lellmann2011}}, based on~{\cite{Chambolle2008}}: Given a metric ({\tmem{distance}}) $d : \{1,
\ldots, l\}^2 \rightarrow \mathbbm{R}_{\geqslant 0}$ (not to be confused with the ambient space dimension), define
\begin{eqnarray}
  & & \Psi_d (z = (z^1, \ldots, z^l)) \assign \sup_{v \in
  \mathcal{D}_{\tmop{loc}}^d}  \langle z, v \rangle, \label{eq:chambreg}\\
  & & \mathcal{D}_{\tmop{loc}}^d \assign \nocomma \{ \left( v^1, \ldots, v^l
  \right) \in \mathbbm{R}^{d \times l} | \ldots\\
  & & \quad \|v^i - v^j \|_2 \leqslant d (i, j)
  \forall i, j \in \{1, \ldots, l\},\ldots\nn\\
  & & \quad \sum_{k = 1}^l v^k = 0\} .\nn
\end{eqnarray}
It was then shown that
\begin{eqnarray}
  \Psi_d (\nu (e^j - e^i)^{\top}) & = & d (i, j),  \label{eq:psidsatisfies}
\end{eqnarray}
therefore in view of (\ref{eq:dpsiduju}) the corresponding regularizer is
{\tmem{non-uniform}}: the boundary between the class regions $\Omega_i$ and
$\Omega_j$ is penalized by its length, multiplied by the weight $d (i, j)$
{\tmem{depending on the labels of both regions}}.

However, even for the comparably simple regularizer (\ref{eq:jusimple}), the
model (\ref{eq:combprob}) is a (spatially continuous)
{\tmem{combinatorial}} problem due to the integral nature of the constraint
set $\mathcal{C}_{\mathcal{E}}$, therefore optimization is nontrivial. In the
context of multiclass image partitioning, a first approach can be found in
{\cite{Lysaker2006}}, where the problem was posed in a level set-formulation
in terms of a labeling function $\phi : \Omega \rightarrow \left\{ 1, \ldots,
l \right\}$, which is subsequently relaxed to~$\mathbbm{R}$. Then $\phi$ is
replaced by polynomials in $\phi$, which coincide with the indicator functions
$u_i$ for the case where $\phi$ assumes integral values. However, the
numerical approach involves several nonlinearities and requires to solve a
sequence of nontrivial subproblems.

In contrast, representation (\ref{eq:combprob}) directly suggests a
more straightforward relaxation to a convex problem: replace $\mathcal{E}$ by
its convex hull, which is just the unit simplex in $l$ dimensions,
\begin{eqnarray}
  \Delta_l & \assign & \tmop{conv} \{e^1, \ldots, e^l \}\\
  & = & \{a \in \mathbbm{R}^l |a \geqslant 0, \sum_{i = 1}^l a_i = 1\},\nonumber
\end{eqnarray}
and solve the {\tmem{relaxed}} problem
\begin{eqnarray}
  \inf_{u \in \mathcal{C}} & f (u) & ,\label{eq:problemrelaxed}\\
  \mathcal{C} & \assign & \tmop{BV} (\Omega, \Delta_l) \\
  & = & \{u \in \tmop{BV} (\Omega)^l |u (x) \in \Delta_l \text{\tmop{~for} a.e.~} x \in \Omega\} \hspace{0.25em} .
\end{eqnarray}
Sparked by a series of papers {\cite{Zach2008,Chambolle2008,Lellmann2009}},
recently there has been much interest in problems of this form, since they --
although generally nonsmooth -- are convex and therefore can be solved to
global optimality, e.g., using primal-dual techniques. The approach has proven
useful for a wide range of applications \cite{Kolev2009,Goldstein2009a,Delaunoy2009,Yuan2010}.

\subsection{Finite-Dimensional vs.~Continuous Approaches}
Many of these applications have been tackled before in a finite-dimensional
setting, where they can be formulated as combinatorial problems on a grid
graph, and solved using combinatorial optimization methods such as
$\alpha$-expansion and related integer linear programming (ILP) methods
{\cite{Boykov2001,Komodakis2007}}. These methods have been shown to yield an
integral labeling $u' \in \mathcal{C}_{\mathcal{E}}$ with the a priori bound
\begin{eqnarray}
  f (u') & \leqslant & 2 \frac{\max_{i \neq j} d (i, j)}{\min_{i \neq j} d (i,
  j)} f (u^{\ast}_{\mathcal{E}}),  \label{eq:fuprime}
\end{eqnarray}
where $u^{\ast}_{\mathcal{E}}$ is the (unknown) solution of the integral
problem (\ref{eq:combprob}). They therefore permit to compute a suboptimal
solution to the -- originally NP-hard {\cite{Boykov2001}} -- combinatorial
problem with an {\tmem{upper bound}} on the objective. No such bound is yet
available for methods based on the spatially continuous problem
(\ref{eq:problemrelaxed}).

Despite these strong theoretical and practical results available for the
finite-dimensional combinatorial energies, the function-based, spatially
continuous formulation (\ref{eq:combprob}) has several unique
advantages:
\begin{itemize}
  \item The energy (\ref{eq:combprob}) is truly isotropic, in the
  sense that for a proper choice of $\Psi$ it is invariant under rotation of
  the coordinate system. Pursuing finite-dimensional ``discretize-first''
  approaches generally introduces artifacts due to the inherent anisotropy,
  which can only be avoided by increasing the neighborhood size, thereby
  reducing sparsity and severely slowing down the graph cut-based methods.
  
  In contrast, properly discretizing the {\tmem{relaxed}} problem
  (\ref{eq:problemrelaxed}) and solving it as a {\tmem{convex}} problem with
  subsequent thresholding yields much better results without compromising
  sparsity (Fig.~\ref{fig:discretization-4nb}
  and~\ref{fig:discretization-iso}, {\cite{Klodt2008}})
  \begin{figure}[tf]
\centering     
    \resizebox{.42\columnwidth}{!}{\includegraphics{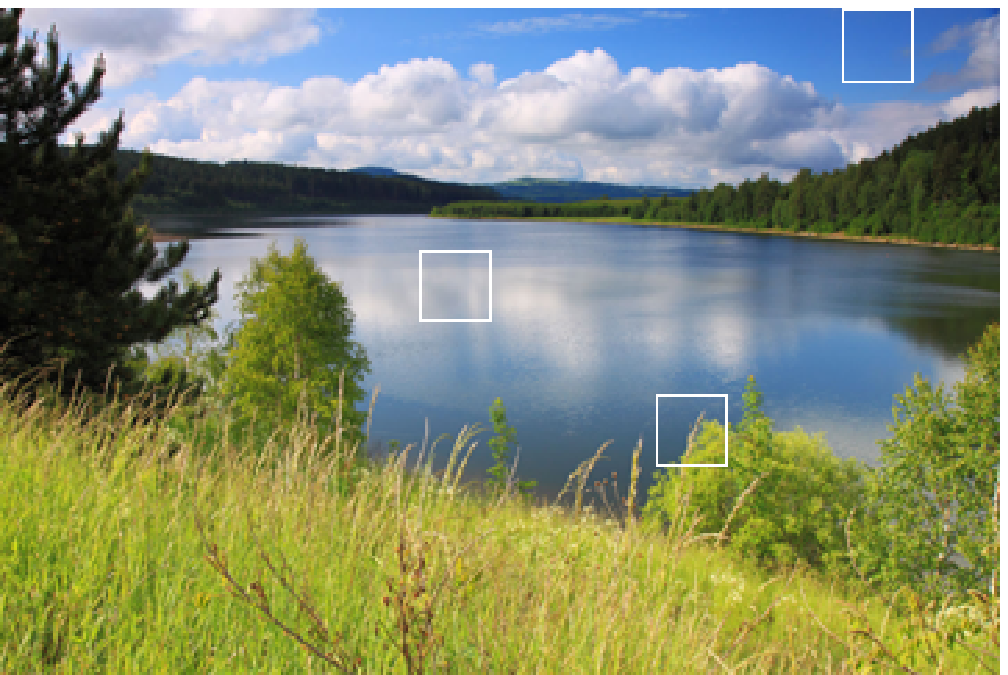}}
    \resizebox{.42\columnwidth}{!}{\includegraphics{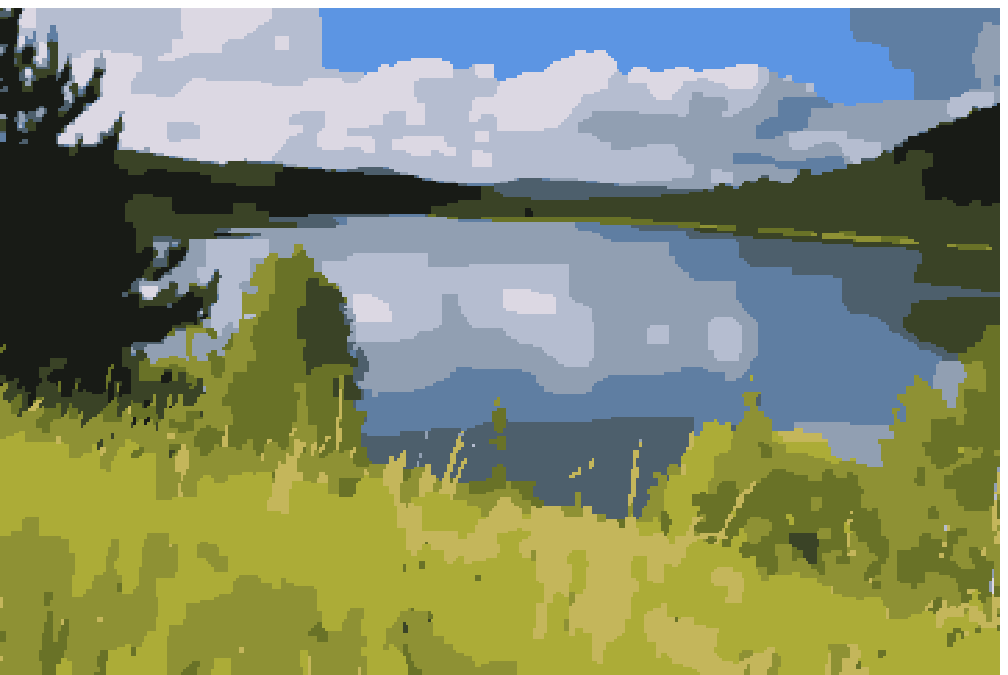}}
    
    \resizebox{.14\columnwidth}{!}{\includegraphics{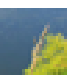}}\resizebox{.14\columnwidth}{!}{\includegraphics{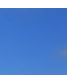}}\resizebox{.14\columnwidth}{!}{\includegraphics{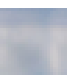}}
 \resizebox{.14\columnwidth}{!}{\includegraphics{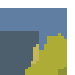}}\resizebox{.14\columnwidth}{!}{\includegraphics{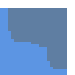}}\resizebox{.14\columnwidth}{!}{\includegraphics{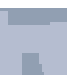}}
    \caption{\label{fig:discretization-4nb}Segmentation of an image into $12$
    classes using a combinatorial method. {\tmstrong{Left:}} Input image,
    {\tmstrong{Right:}} Result obtained by solving a
    {\tmem{combinatorial}} discretized problem with $4$-neighborhood. The
    bottom row shows detailed views of the marked parts of the image. The
    minimizer of the combinatorial problem exhibits blocky artifacts caused by
    the choice of discretization.}
  \end{figure}.\begin{figure}[tf]
  \centering
    \resizebox{.42\columnwidth}{!}{\includegraphics{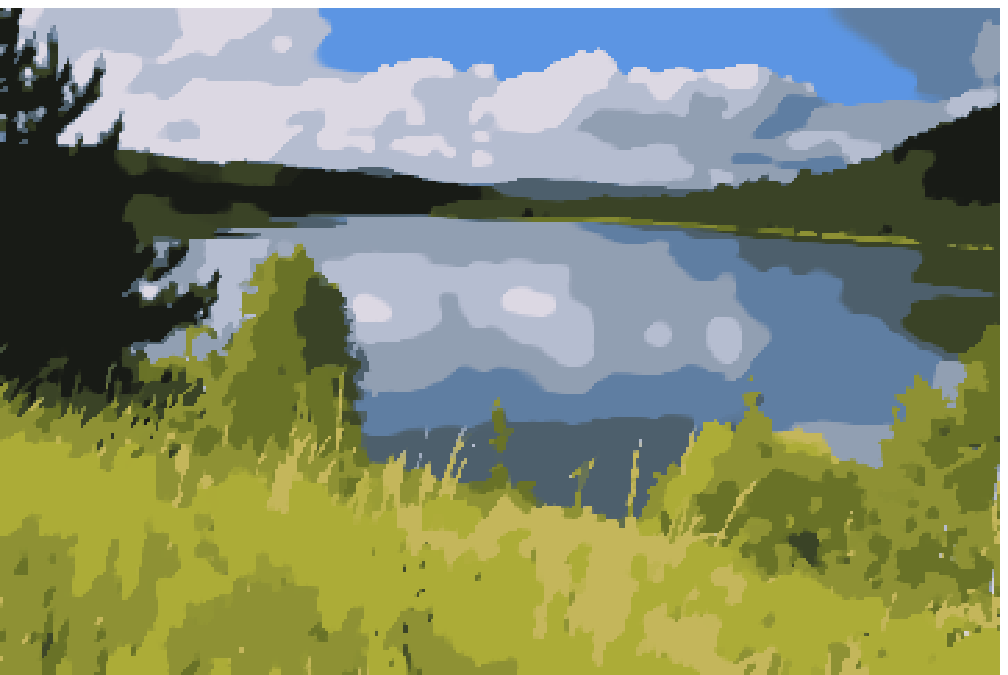}}
    \resizebox{.42\columnwidth}{!}{\includegraphics{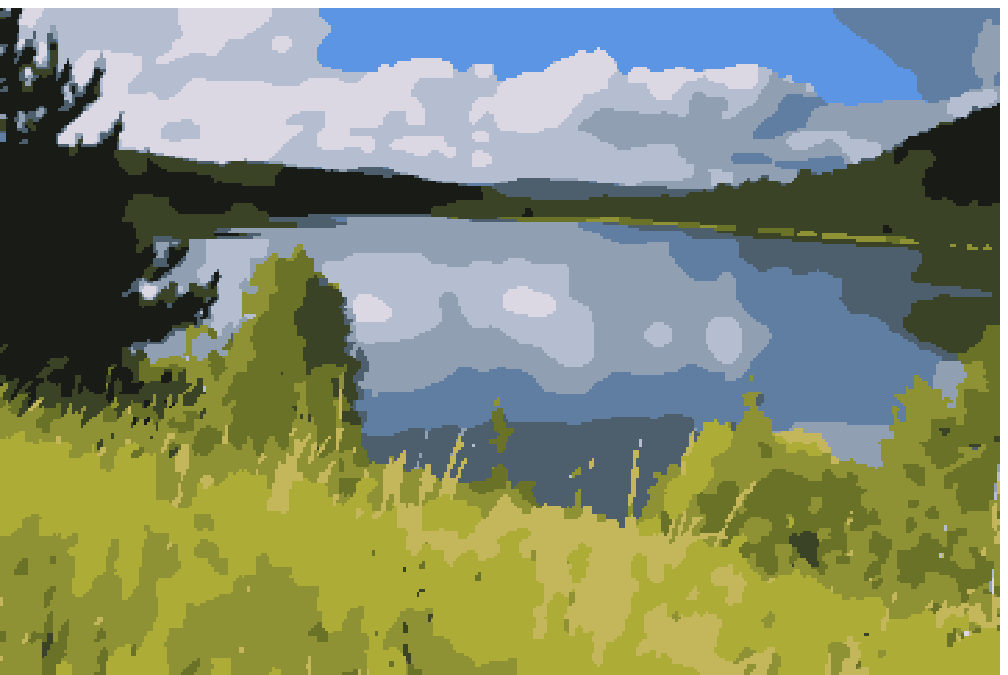}}
    
    \resizebox{.14\columnwidth}{!}{\includegraphics{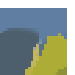}}\resizebox{.14\columnwidth}{!}{\includegraphics{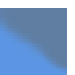}}\resizebox{.14\columnwidth}{!}{\includegraphics{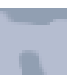}}
    \resizebox{.14\columnwidth}{!}{\includegraphics{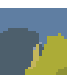}}\resizebox{.14\columnwidth}{!}{\includegraphics{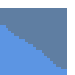}}\resizebox{.14\columnwidth}{!}{\includegraphics{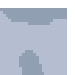}}
    \caption{\label{fig:discretization-iso}Segmentation obtained by solving a
    finite-differences discretization of the {\tmem{relaxed}} spatially continuous problem.
    {\tmstrong{Left:}} Non-integral solution obtained as a minimizer of the
    discretized relaxed problem. {\tmstrong{Right:}} Integral labeling
    obtained by rounding the fractional labels in the solution of the relaxed
    problem to the nearest integral label. The rounded result contains almost
    no structural artifacts.}
  \end{figure} This can be attributed to the fact that solving the discretized
  problem as a {\tmem{combinatorial}} problem in effect discards much of the
  information about the problem structure that is contained in the nonlinear
  terms of the discretized objective.
  
  \item Present combinatorial optimization methods
  {\cite{Boykov2001,Komodakis2007}} are inherently sequential and difficult to
  parallelize. On the other hand, parallelizing primal-dual methods for
  solving the relaxed problem \eqref{eq:problemrelaxed} is straight-forward,
  and GPU implementations have been shown to outperform state-of-the-art graph
  cut methods~\cite{Zach2008}.
  
  \item Analyzing the problem in a fully functional-analytic setting gives
  valuable insight into the problem structure, and is of theoretical interest
  in itself.
\end{itemize}

\subsection{Optimality Bounds}
However, one possible drawback of the spatially continuous approach is that
the solution of the relaxed problem (\ref{eq:problemrelaxed}) may assume
{\tmem{fractional}} values, i.e.,~values in $\Delta_l \setminus \mathcal{E}$.
Therefore, in applications that require a true partition of $\Omega$, some
rounding process is needed in order to generate an integral labeling
$\bar{u}^{\ast}$. This may increase the objective, and lead to a
suboptimal solution of the original problem~(\ref{eq:combprob}).

The regularizer $\Psi_d$ as defined in (\ref{eq:chambreg}) enjoys the property
that it majorizes all other regularizers that can be written in integral form
and satisfy (\ref{eq:psidsatisfies}). Therefore it is in a sense ``optimal'',
since it introduces as few fractional solutions as possible. In practice, this
forces solutions of the relaxed problem to assume integral values in most
points, and rounding is in practice only required in small regions.

However, the rounding step may still increase the objective and generate
suboptimal integral solutions. Therefore the question arises whether this
approach allows to recover ``good'' integral solutions of the original problem
(\ref{eq:combprob}).

In the following, we are concerned with the question whether it is possible to
obtain, using the convex relaxation (\ref{eq:problemrelaxed}), {\tmem{integral}}
solutions with an upper bound on the objective. Specifically, we focus
on inequalities of the form
\begin{eqnarray}
  f ( \bar{u}^{\ast}) & \leqslant & (1 + \varepsilon) f
  (u_{\mathcal{E}}^{\ast})  \label{eq:subopteps}
\end{eqnarray}
for some constant $\varepsilon > 0$, which provide an {\tmem{upper bound on
the objective}} of the \tmtextit{rounded} integral solution $\bar{u}^{\ast}$
with respect to the objective of the (unknown) {\tmem{optimal}} integral
solution $u_{\mathcal{E}}^{\ast}$ of (\ref{eq:combprob}). Note that generally
it is not possible to show that (\ref{eq:subopteps}) holds for
{\tmem{any}} $\varepsilon > 0$. The reverse inequality
\begin{eqnarray}
  f (u_{\mathcal{E}}^{\ast}) \leqslant f ( \bar{u}^{\ast}) 
\end{eqnarray}
always holds since $\bar{u}^{\ast} \in \mathcal{C}_{\mathcal{E}}$ and
$u_{\mathcal{E}}^{\ast}$ is an optimal integral solution. The specific form
(\ref{eq:subopteps}) can be attributed to the alternative interpretation
\begin{eqnarray}
  \frac{f ( \bar{u}^{\ast}) - f (u_{\mathcal{E}}^{\ast})}{f
  (u_{\mathcal{E}}^{\ast})} & \leqslant & \varepsilon, 
  \label{eq:fubarboundalt}
\end{eqnarray}
which provides a bound for the relative gap to the optimal objective of the
combinatorial problem. Such $\varepsilon$ can be obtained {\tmem{a
posteriori}} by actually computing (or approximating)~$\bar{u}^{\ast}$ and a
dual feasible point: Assume that a feasible primal-dual pair $\left( u, v
\right) \in \mathcal{C} \times \mathcal{D}$ is known, where $u$ approximates
$u^{\ast}$, and assume that some integral feasible $\bar{u} \in
\mathcal{C}_{\mathcal{E}}$ has been obtained from $u$ by a rounding process.
Then the pair $\left( \bar{u}, v \right)$ is feasible as well since
$\mathcal{C}_{\mathcal{E}} \subseteq \mathcal{C}$, and we obtain an {\tmem{a posteriori}}
optimality bound of the form (\ref{eq:fubarboundalt}) with respect to the
optimal \tmtextit{integral} solution $u^{\ast}_{\mathcal{E}}$:
\begin{eqnarray}
  & \frac{f ( \bar{u}) - f_D (u_{\mathcal{E}}^{\ast})}{f_D
  (u_{\mathcal{E}}^{\ast})} \leqslant \frac{f ( \bar{u}) - f_D
  (u_{\mathcal{E}}^{\ast})}{f_D (v)} \leqslant \frac{f ( \bar{u}) - f_D
  (v)}{f_D (v)} = : \varepsilon' \hspace{0.25em} . &  \label{eq:apost}
\end{eqnarray}
However, this requires that the the primal and dual objectives $f$ and $f_D$
can be accurately evaluated, and requires to compute a minimizer of the
problem for the specific input data, which is generally difficult, especially
in the spatially continuous formulation.

In contrast, true {\tmem{a priori}} bounds do not require knowledge of a
solution and apply uniformly to all problems of a class, irrespective of the
particular input. When considering rounding methods, one generally has to
discriminate between
\begin{itemize}
  \item {\tmem{deterministic}} vs.~{\tmem{probabilistic}} methods, and
  
  \item {\tmem{spatially discrete (finite-dimensional)}} vs.~{\tmem{spatially
  continuous}} methods.
\end{itemize}
Most known {\tmem{a priori}} approximation results only hold in the
finite-dimensional setting, and are usually proven using graph-based pairwise
formulations. In contrast, we assume an ``optimize first'' perspective due to
the reasons outlined in the introduction. Unfortunately, the proofs for the
finite-dimensional results often rely on pointwise arguments that cannot
directly be transferred to the continuous setting. Deriving similar results
for continuous problems therefore requires considerable additional work.

\subsection{Contribution and Main Results}
In this work we prove that using the regularizer
(\ref{eq:chambreg}), the {\tmem{a priori}} bound (\ref{eq:fuprime}) can be
carried over to the spatially continuous setting. Preliminary versions of
these results with excerpts of the proofs have been announced as conference
proceedings {\cite{Lellmann2011}}. We extend these results to provide the
exact bound (\ref{eq:fuprime}), and supply the full proofs.

As the main result, we show that it is possible to construct a rounding method
parametrized by a parameter $\gamma \in \Gamma$, where $\Gamma$ is an
appropriate parameter space:
\begin{eqnarray}
  & R : & \mathcal{C} \times \Gamma \rightarrow \mathcal{C}_{\mathcal{E}}, \\
  &  & u \in \mathcal{C} \mapsto \bar{u}_{\gamma} \assign R_{\gamma} (u) \in
  \mathcal{C}_{\mathcal{E}},
\end{eqnarray}
such that for a suitable probability distribution on $\Gamma$, the following
theorem holds for the expectation $\mathbbm{E}f ( \bar{u}) \assign
\mathbbm{E}_{\gamma} f ( \bar{u}_{\gamma})$:

\begin{theorem}
  \label{thm:mainthm}Let $u \in \mathcal{C}$, $s \in L^1 (\Omega)^l$, $s
  \geqslant 0$, and let $\Psi : \mathbbm{R}^{d \times l} \rightarrow
  \mathbbm{R}_{\geqslant 0}$ be positively homogeneous, convex and continuous.
  Assume there exists a lower bound $\lambda_l > 0$ such that, for $z = (z^1, \ldots, z^l)$,
  \begin{eqnarray}
    \Psi (z) & \geqslant & \lambda_l  \frac{1}{2} 
    \sum_{i = 1}^l \|z^i \|_2 \hspace{1em} \forall z \in \mathbbm{R}^{d \times
    l}, \sum_{i = 1}^l z^i = 0.  \label{eq:psilower}
  \end{eqnarray}
  Moreover, assume there exists an upper bound $\lambda_u < \infty$ such that, for every $\nu \in \mathbbm{R}^d satisfying  \| \nu \|_2 = 1$,
  \begin{eqnarray}
    \Psi (\nu (e^i - e^j)^{\top}) & \leqslant & \lambda_u \quad \forall i, j \in \{1, \ldots, l\}\,.\label{eq:psiupper}
  \end{eqnarray}
  Then Alg.~1 (see below) generates an integral
  labeling $\bar{u} \in \mathcal{C}_{\mathcal{E}}$ almost surely, and
  \begin{eqnarray}
    \mathbbm{E} f ( \bar{u}) & \leqslant & 2 \frac{\lambda_u}{\lambda_l} f (u)
    .  \label{eq:efubaroptimality}
  \end{eqnarray}
\end{theorem}

Note that $\lambda_u \geqslant \lambda_l$ always holds if both are defined, since
\eqref{eq:psilower} implies, for $\nu$ with $\| \nu \|_2 = 1$,
\begin{eqnarray}
  \lambda_u \geqslant \Psi (\nu (e^i - e^j)^{\top}) & \geqslant &
  \frac{\lambda_l}{2} ( \| \nu \|_2 + \| \nu \|_2) = \lambda_l . 
\end{eqnarray}
The proof of Thm.~\ref{thm:mainthm} (Sect.~\ref{sec:proofmainthm}) is based on the work of Kleinberg and Tardos {\cite{Kleinberg1999}},
which is set in an LP relaxation framework. However their results are
restricted in that they assume a graph-based representation and extensively
rely on the finite dimensionality. In contrast, our results hold in the
continuous setting without assuming a particular problem discretization.

Theorem~\ref{thm:mainthm} guarantees that -- in a probabilistic sense -- the
rounding process may only increase the energy in a controlled way, with an
upper bound depending on~$\Psi$. An immediate consequence is

\begin{corollary}
  \label{cor:maincor}Under the conditions of Thm.~\ref{thm:mainthm}, if
  $u^{\ast}$ minimizes $f$ over $\mathcal{C}$, $u_{\mathcal{E}}^{\ast}$
  minimizes $f$ over $\mathcal{C}_{\mathcal{E}}$, and $\bar{u}^{\ast}$ denotes
  the output of Alg.~1 applied to $u^{\ast}$, then
  \begin{eqnarray}
    \mathbbm{E} f \left( \bar{u}^{\ast} \right) & \leqslant & 2
    \frac{\lambda_u}{\lambda_l} f (u_{\mathcal{E}}^{\ast}) . 
    \label{eq:efubarstar}
  \end{eqnarray}
\end{corollary}

Therefore the proposed approach allows to recover, from the solution
$u^{\ast}$ of the convex {\tmem{relaxed}} problem (\ref{eq:problemrelaxed}), an
approximate \tmtextit{integral} solution $\bar{u}^{\ast}$ of the nonconvex
{\tmem{original}} problem (\ref{eq:combprob}) with an upper bound on the
objective.

In particular, for the tight relaxation of the regularizer as in
(\ref{eq:chambreg}), we obtain (cf.~Prop.~\ref{prop:dijprop})
\begin{eqnarray}
  \mathbbm{E}f ( \bar{u}^{\ast}) & \leqslant & 2 \frac{\lambda_u}{\lambda_l} =
  2 \frac{\max_{i \neq j} d (i, j)}{\min_{i \neq j} d (i, j)}, 
  \label{eq:psidclaim}
\end{eqnarray}
which is exactly the same bound as has been proven for the combinatorial
$\alpha$-expansion method (\ref{eq:fuprime}).

To our knowledge, this is the first bound available for the fully spatially
convex relaxed problem (\ref{eq:problemrelaxed}). Related is the work of
Olsson et al.~{\cite{Olsson2009,Olsson2009a}}, where the authors consider a
continuous analogue to the $\alpha$-expansion method known as continuous
binary fusion {\cite{Trobin2008}}, and claim that a bound similar to
(\ref{eq:fuprime}) holds for the corresponding fixed points when using the
separable regularizer
\begin{eqnarray}
  \Psi_A (z) & \assign & \sum_{j = 1}^l \|Az^j_{} \|_2, \hspace{1em} z \in
  \mathbbm{R}^{d \times l}, 
\end{eqnarray}
for some $A \in \mathbbm{R}^{d \times d}$, which implements an anisotropic
variant of the uniform metric. However, a rigorous proof in the BV framework
was not given.

In {\cite{Bae2011}}, the authors propose to solve the problem
(\ref{eq:combprob}) by considering the dual problem to
(\ref{eq:problemrelaxed}) consisting of $l$ coupled maximum-flow problems,
which are solved using a log-sum-exp smoothing technique and gradient descent.
In case the dual solution allows to unambiguously recover an (integral) primal
solution, the latter is necessarily the unique minimizer of $f$, and therefore
a global {\tmem{integral}} minimizer of the combinatorial problem
(\ref{eq:combprob}). This provides an {\tmem{a posteriori}} bound,
which applies if a dual solution can be computed. While useful in practice as
a certificate for global optimality, in the spatially continuous setting it
requires explicit knowledge of a dual solution, which is rarely available
since it depends on the regularizer $\Psi$ as well as the input data~$s$.

In contrast, the {\tmem{a priori}} bound (\ref{eq:efubarstar}) holds uniformly
over all problem instances, does not require knowledge of any primal or dual
solutions and covers also non-uniform regularizers.

\section{A Probabilistic View of the Coarea Formula}

\subsection{The Two-Class Case}

As a motivation for the following sections, we first provide a probabilistic
interpretation of a tool often used in geometric measure theory, the
{\tmem{coarea formula}} (cf.~{\cite{Ambrosio2000}}). Assuming $u' \in
\tmop{BV} (\Omega)$ and $u' (x) \in [0, 1]$ for a.e.~$x \in \Omega$, the
coarea formula states that the total variation of $u$ can be represented by summing
the boundary lengths of its super-levelsets:
\begin{eqnarray}
  \tmop{TV} (u') & = & \int_0^1 \tmop{TV} (1_{\{u' > \alpha\}}) d \alpha_{} . 
  \label{eq:probcoareaintro}
\end{eqnarray}
Here $1_A$ denotes the characteristic function of a set $A$, i.e.,~$1_A (x) =
1$ iff $x \in A$ and $1_A (x) = 0$ otherwise. The coarea formula provides a
connection between problem~(\ref{eq:combprob}) and the
relaxation~(\ref{eq:problemrelaxed}) in the two-class case, where $\mathcal{E} =
\{e^1, e^2 \}$, $u \in \mathcal{C}_{\mathcal{E}}$ and $u_1 = 1 - u_2$: As
noted in {\cite{Lellmann2009a}},
\begin{eqnarray}
  & \tmop{TV} (u) = \|e^1 - e^2 \|_2 \tmop{TV} (u_1) = \sqrt{2} \tmop{TV}
  (u_1), & 
\end{eqnarray}
therefore the coarea formula~(\ref{eq:probcoareaintro}) can be rewritten as
\begin{eqnarray}
  \tmop{TV} (u) & = & \sqrt{2} \int_0^1 \tmop{TV} (1_{\{u_1 > \alpha\}}) d
  \alpha \\
  & = & \int_0^1 \tmop{TV} (e^1 1_{\{u_1 > \alpha\}} + e^2 1_{\{u_1 \leqslant
  \alpha\}}) d \alpha \\
  & = & \int_0^1 \tmop{TV} ( \bar{u}_{\alpha}) d \alpha,\\
  & & \bar{u}_{\alpha} \assign e^1 1_{\{u_1 > \alpha\}} + e^2
  1_{\{u_1 \leqslant \alpha\}} .  \label{eq:tvbarycenter}
\end{eqnarray}
Consequently, the total variation of $u$ can be expressed as the {\tmem{mean}}
over the total variations of a set of {\tmem{integral}} labelings $\{
\bar{u}_{\alpha} \in \mathcal{C}_{\mathcal{E}} | \alpha \in [0, 1]\}$,
obtained by {\tmem{rounding $u$ at different thresholds $\alpha$}}. We now
adopt a {\tmem{probabilistic}} view of~(\ref{eq:tvbarycenter}): We regard the~mapping
\begin{eqnarray}
  R : (u, \alpha) \in \mathcal{C} \times [0, 1] & \mapsto & \bar{u}_{\alpha}
  \in \mathcal{C}_{\mathcal{E}} \quad (\text{a.e.~} \alpha
  \in [0, 1])  \label{eq:probparamround}
\end{eqnarray}
as a {\tmem{parametrized, deterministic}} rounding algorithm that depends on
$u$ and on an additional parameter $\alpha$. From this we obtain a
\tmtextit{probabilistic} (randomized) rounding algorithm by assuming $\alpha$
to be a uniformly distributed random variable. Under these assumptions the
coarea formula (\ref{eq:tvbarycenter}) can be written as
\begin{eqnarray}
  \tmop{TV} (u) & = & \mathbbm{E}_{\alpha} \tmop{TV} ( \bar{u}_{\alpha}) . 
  \label{eq:coareaasexpectation}
\end{eqnarray}
This has the probabilistic interpretation that applying the probabilistic
rounding to (arbitrary, but fixed) $u$ does -- in a probabilistic sense,
i.e.,~in the mean -- not change the objective. It can be shown that this
property extends to the full functional $f$ in (\ref{eq:problemrelaxed}): In the
two-class case, the ``coarea-like'' property
\begin{eqnarray}
  f (u) & = & \mathbbm{E}_{\alpha} f ( \bar{u}_{\alpha}) 
  \label{eq:coarealike}
\end{eqnarray}
holds. Functions with property (\ref{eq:coarealike}) are also known as
{\tmem{levelable functions}} {\cite{Darbon2006,Darbon2006a}} or
{\tmem{discrete total variations}}~{\cite{Chambolle2009}} and have been
studied in {\cite{Strandmark2009}}. A well-known implication is that if $u =
u^{\ast}$, i.e.,~$u$ minimizes the relaxed problem (\ref{eq:problemrelaxed}), then in
the two-class case almost every $\bar{u}^{\ast} = \bar{u}^{\ast}_{\alpha}$ is
an integral minimizer of the original problem (\ref{eq:combprob}),
i.e.,~the optimality bound (\ref{eq:subopteps}) holds with $\varepsilon = 0$
{\cite{Nikolova2006}}.

\subsection{The Multi-Class Case and Generalized Coarea Formulas}

Generalizing these observations to more than two labels hinges on a property
similar to (\ref{eq:coarealike}) that holds for {\tmem{vector-valued}}
$u$. In a general setting, the question is whether there exist
\begin{itemize}
  \item a probability space $(\Gamma, \mu)$, and
  
  \item a {\tmem{parametrized rounding method}}, i.e.,~for $\mu$-almost every~$\gamma
  \in \Gamma$:
  \begin{eqnarray}
    & R : & \mathcal{C} \times \Gamma \rightarrow \mathcal{C}_{\mathcal{E}},
    \\
    &  & u \in \mathcal{C} \mapsto \bar{u}_{\gamma} \assign R_{\gamma} (u)
    \in \mathcal{C}_{\mathcal{E}} 
  \end{eqnarray}
  satisfying $R_{\gamma} (u') = u'$ for all $u' \in
  \mathcal{C}_{\mathcal{E}}$,
\end{itemize}
such that a ``multiclass coarea-like property'' (or {\tmem{generalized coarea
formula}})
\begin{eqnarray}
  f (u) & = & \int_{\Gamma} f ( \bar{u}_{\gamma}) d \mu (\gamma) 
  \label{eq:coarea-like-like}
\end{eqnarray}
holds. In a probabilistic sense this corresponds to
\begin{eqnarray}
  f (u) & = & \int_{\Gamma} f ( \bar{u}_{\gamma}) d \mu (\gamma) =
  \mathbbm{E}_{\gamma} f ( \bar{u}_{\gamma}).  \label{eq:probroundgeneralized}
\end{eqnarray}
For $l = 2$ and $\Psi (x) = \|\cdummy \|_2$, (\ref{eq:coareaasexpectation}) shows that
(\ref{eq:probroundgeneralized}) holds with $\gamma = \alpha$, $\Gamma = [0,
1]$, $\mu =\mathcal{L}^1$, and $R : \mathcal{C} \times \Gamma \rightarrow
\mathcal{C}_{\mathcal{E}}$ as defined in (\ref{eq:probparamround}). Unfortunately, property (\ref{eq:coareaasexpectation}) is intrinsically
restricted to the two-class case with $\tmop{TV}$ regularizer.

In the multiclass case, the difficulty lies in providing a suitable
combination of a probability space $(\Gamma, \mu)$ and a parametrized rounding
step $(u, \gamma) \mapsto \bar{u}_{\gamma}$. Unfortunately, obtaining a
relation such as (\ref{eq:coareaasexpectation}) for the full functional
(\ref{eq:combprob}) is unlikely, as it would mean that solutions to the (after
discretization) NP-hard problem (\ref{eq:combprob}) could be obtained by
solving the convex relaxation (\ref{eq:problemrelaxed}) and subsequent rounding, which can be achieved in polynomial time.

Therefore we restrict ourselves to an \tmtextit{approximate} variant of the
generalized coarea formula:
\begin{eqnarray}
  (1 + \varepsilon) f (u) & \geqslant & \int_{\Gamma} f ( \bar{u}_{\gamma}) d
  \mu (\gamma) = \mathbbm{E}_{\gamma} f ( \bar{u}_{\gamma}) . 
  \label{eq:probrounddesired}
\end{eqnarray}
While (\ref{eq:probrounddesired}) is not sufficient to provide a bound on $f (
\bar{u}_{\gamma})$ for {\tmem{particular}} $\gamma$, it permits a
{\tmem{probabilistic}} bound: for any minimizer $u^{\ast}$ of the relaxed problem (\ref{eq:problemrelaxed}), eq.~\eqref{eq:probrounddesired} implies
\begin{eqnarray}
  & \mathbbm{E}_{\gamma} f ( \bar{u}^{\ast}_{\gamma}) \leqslant (1 +
  \varepsilon) f (u^{\ast}) \leqslant (1 + \varepsilon) f
  (u^{\ast}_{\mathcal{E}}), &  \label{eq:proboptimalityfromcoarea}
\end{eqnarray}
i.e.,~the ratio between the objective of the {\tmem{rounded relaxed
solution}} and the {\tmem{optimal integral solution}} is bounded -- in a
probabilistic sense -- by $(1 + \varepsilon)$.

In the following sections we construct a suitable parametrized rounding method
and probability space in order to obtain an approximate generalized coarea formula of the
form (\ref{eq:probrounddesired}).

\section{Probabilistic Rounding for Multiclass Image Partitions}\label{sec:aprioribounds}

\subsection{Approach}

We consider the probabilistic rounding approach based on
{\cite{Kleinberg1999}} as defined in
Alg.~1. %

\begin{algorithm}[tb]
\caption{Continuous Probabilistic Rounding}\label{alg:contprobround}
\begin{algorithmic}[1]
  \STATE $u^0 \leftarrow u$, $U^0 \leftarrow \Omega$, $c^0 \leftarrow (1, \ldots,1) \in \mathbbm{R}^l$.
    
  \FOR{$k = 1, 2, \ldots$}
   
    \STATE \label{step:choosealphak}Randomly choose $\gamma^k = (i^k, \alpha^k) \in \mathcal{I} \times [0, 1]$ uniformly.

    \STATE $M^k \leftarrow U^{k - 1} \cap \{^{} x \in \Omega |u_{i^k}^{k - 1} (x) > \alpha^k \}$.
      
    \STATE \label{step:assignu}$u^k \leftarrow e^{i^k} 1_{M^k} + u^{k - 1} 1_{\Omega \setminus M^k}$.
    
    \STATE $U^k \leftarrow U^{k - 1} \setminus M^k$.
      
    \STATE $c_j^k \leftarrow \left\{ \begin{array}{ll}
        \min \{c_j^{k - 1}, \alpha^k \}, & j = i^k,\\
        c_j^{k - 1}, & \tmop{otherwise} .
      \end{array} \right.$
  \ENDFOR
\end{algorithmic}
\end{algorithm}

The algorithm proceeds in a number of phases. At each iteration,
a label and a threshold
\begin{eqnarray*}
  \gamma^k \assign (i^k, \alpha^k) & \in & \Gamma' \assign \mathcal{I} \times
  [0, 1]
\end{eqnarray*}
are randomly chosen (step \ref{step:choosealphak}), and label $i^k$ is
assigned to all yet unassigned points~$x$ where $u_{i^k}^{k - 1} (x) >
\alpha^k$ holds (step \ref{step:assignu}). In contrast to the two-class case
considered above, the randomness is provided by a \tmtextit{sequence}
$(\gamma^k)$ of uniformly distributed random variables, i.e.,~$\Gamma =
(\Gamma')^{\mathbbm{N}}$.

After iteration $k$, all points in the set $U^k \subseteq \Omega$ are still
{\tmem{unassigned}}, while all points in $\Omega \setminus U^k$ have been
assigned an (integral) label in iteration $k$ or in a previous iteration. 
Iteration $k + 1$ potentially modifies points only in the set $U^k$. The
variable $c_j^k$ stores the lowest threshold $\alpha$ chosen for label $j$ up
to and including iteration $k$, and is only required for the proofs.

While the algorithm is defined using pointwise operations, it is well-defined
in the sense that for fixed $\gamma$, the sequence $(u^k)$, viewed as elements
in $L^1$, does not depend on the specific representative of $u$ in its equivalence
class in $L^1$. The sequences $(M^k)$ and $(U^k)$ depend on the
representative, but are unique up to $\mathcal{L}^d$-negligible sets.

In an actual implementation, the algorithm could be terminated as soon as all
points in $\Omega$ have been assigned a label, i.e.,~$U^k = \emptyset$.
However, in our framework used for analysis the algorithm never terminates
explicitly. Instead, for fixed input $u$ we regard the algorithm as a mapping
between {\tmem{sequences}} of parameters (or instances of random variables)
$\gamma = (\gamma^k) \in \Gamma$ and {\tmem{sequences}} of states
$(u_{\gamma}^k)$, $(U_{\gamma}^k)$ and $(c_{\gamma}^k)$. We drop the
subscript~$\gamma$ if it does not create ambiguities. The elements of the
sequence $(\gamma^{(k)})$ are independently uniformly distributed, and by the
Kolmogorov extension theorem {\cite[Thm. 2.1.5]{Oksendal2003}} there exists a
probability space and a stochastic process on the set of sequences $\gamma$
with compatible marginal distributions.

In order to define the parametrized rounding step $(u, \gamma) \mapsto
\bar{u}_{\gamma}$, we observe that once $U^{k'}_{\gamma} = \emptyset$ occurs
for some $k' \in \mathbbm{N}$, the sequence $(u^k_{\gamma})$ becomes
stationary at $u_{\gamma}^{k'}$. In this case the algorithm may be terminated,
with output $\bar{u}_{\gamma} \assign u_{\gamma}^{k'}$:

\begin{definition}
  \label{def:fubar}Let $u \in \tmop{BV} (\Omega)^l$ and $f : \tmop{BV}
  (\Omega)^l \rightarrow \mathbbm{R}$. For some $\gamma \in \Gamma$, if
  $U_{\gamma}^{k'} = \emptyset$ in Alg.~1 for some
  $k' \in \mathbbm{N}$, we denote $\bar{u}_{\gamma} \assign u_{\gamma}^{k'}$.
  We define
  \begin{eqnarray}
    & & f ( \bar{u}_{(\cdummy)}) : \Gamma \rightarrow \mathbbm{R} \cup \{+
    \infty\}, \gamma \in \Gamma \mapsto f ( \bar{u}_{\gamma})\,, \label{eq:fugn}\\
    & & f ( \bar{u}_{\gamma}) \assign \left\{
    \begin{array}{ll}
      f (u^{k'}_{\gamma}), & \exists k' \in \mathbbm{N}: \; U^{k'}_{\gamma} = \emptyset \wedge u^{k'}_{\gamma} \in \tmop{BV}
      (\Omega)^l,\\
      + \infty, & \text{\tmop{otherwise}} .
    \end{array} \right.\nn
  \end{eqnarray}
  We denote by $f ( \bar{u})$ the corresponding random variable induced by
  assuming $\gamma$ to be uniformly distributed on~$\Gamma$.
\end{definition}

As indicated above, $f ( \bar{u}_{\gamma})$ is well-defined: if
$U^{k'}_{\gamma} = \emptyset$ for some $(\gamma, k')$ then $u_{\gamma}^{k'} =
u_{\gamma}^{k''}$ for all $k'' \geqslant k'$. Instead of focusing on local
properties of the random sequence $(u_{\gamma}^k)$ as in the proofs for the
finite-dimensional case, we derive our results directly for the sequence $(f
(u_{\gamma}^k))$. In particular, we show that the expectation of $f (
\bar{u})$ over all sequences $\gamma$ can be bounded according to
\begin{eqnarray}
  \mathbbm{E} f ( \bar{u}) = \mathbbm{E}_{\gamma} f ( \bar{u}_{\gamma}) &
  \leqslant & (1 + \varepsilon) f ( \bar{u})  \label{eq:thegoal}
\end{eqnarray}
for some $\varepsilon \geqslant 0$, cf.~(\ref{eq:probrounddesired}).
Consequently, the rounding process may only increase the average objective in
a controlled way.

\subsection{Termination Properties}\label{sec:probroundterm}

Theoretically, the algorithm may produce a sequence $(u_{\gamma}^k)$ that does
{\tmem{not}} become stationary, or becomes stationary with a solution that is
not an element of $\tmop{BV} (\Omega)^l$. In Thm.~\ref{thm:probterminates} below we
show that this happens only with zero probability, i.e.,~almost surely
Alg.~1 generates (in a finite number of iterations)
an \tmtextit{integral} labeling function $\bar{u}_{\gamma} \in
\mathcal{C}_{\mathcal{E}}$. The following two propositions are required for
the proof.

\begin{proposition}
  \label{prop:petckl1}For the sequence $(c^k)$ generated by
  Algorithm~1,
  \begin{eqnarray}
    & & \mathbbm{P} (e^{\top} c^k < 1) \geqslant \label{eq:convprop1}\\
    & & \quad \sum_{p \in \{0, 1\}^l}
    \left( - 1 \right)^{e^{\top} p}  \left( \sum_{j = 1}^l \frac{1}{l}  \left(
    \left( 1 - \frac{1}{l} \right)^{p_j} \right) \right)^k\nonumber    
  \end{eqnarray}
  holds. In particular,
  \begin{eqnarray}
    \mathbbm{P} (e^{\top} c^k < 1) & \overset{k \rightarrow
    \infty}{\rightarrow} & 1.  \label{eq:convprop2}
  \end{eqnarray}
\end{proposition}

\begin{proof}
  Denote by $n^k_j \in \mathbbm{N}_0$ the number of $k' \in \{1, \ldots, k\}$
  such that $i^{k'} = j$, i.e.,~the number of times label $j$ was selected up
  to and including the $k$-th step. Then
  \begin{eqnarray}
    (n_1^k, \ldots, n_l^k) & \sim & \tmop{Multinomial} \left( k ; \frac{1}{l},
    \ldots, \frac{1}{l} \right), 
  \end{eqnarray}
  i.e.,~the probability of a specific instance is
  \begin{eqnarray}
    \mathbbm{P} ((n_1^k, \ldots, n_l^k)) & = & \left\{ \begin{array}{ll}
      \frac{k!}{n_1^k ! \cdot \ldots \cdot n_l^k !}  \left( \frac{1}{l}
      \right)^k, & \sum_j n_j^k = k,\\
      0, & \tmop{otherwise} .
    \end{array} \right. 
  \end{eqnarray}
  Therefore,
  \begin{eqnarray}
    \mathbbm{P} (e^{\top} c^k < 1) & = & \sum_{n_1^k, \ldots, n_l^k}
    \mathbbm{P}(e^{\top} c^k < 1| (n_1^k, \ldots, n_l^k))\cdot\ldots\nonumber\\
    & & \quad\quad \mathbbm{P}((n_1^k, \ldots, n_l^k)) \\
    & = & \sum_{n_1^k + \ldots + n_l^k = k} \frac{k!}{n_1^k ! \cdot \ldots \cdot n_l^k !} \left( \frac{1}{l} \right)^k\cdot\ldots\nonumber\\
    & & \quad\quad \mathbbm{P}(e^{\top} c^k < 1|
    (n_1^k, \ldots, n_l^k)) .  \label{eq:etctoshow}
  \end{eqnarray}
  Since $c_1^k, \ldots, c_l^k < \frac{1}{l}$ is a sufficient condition for
  $e^{\top} c < 1$, we may bound the probability according to
  \begin{eqnarray}
    \mathbbm{P} (e^{\top} c < 1) & \geqslant & \sum_{n_1^k + \ldots + n_l^k =
    k} \frac{k!}{n_1^k ! \cdot \ldots \cdot n_l^k !}  \left( \frac{1}{l}
    \right)^k \cdot\ldots\nonumber\\
    & & \quad\quad \mathbbm{P} \left( c_j^k < \frac{1}{l} \forall j \in
    \mathcal{I}| (n_1^k, \ldots, n_l^k) \right).  \label{eq:petc1}
  \end{eqnarray}
  We now consider the distributions of the components $c^k_j$ of $c^k$
  conditioned on the vector $(n_1^k, \ldots, n_l^k)$. Given $n_j^k$, the
  probability of $\{c_j^k \geqslant t\}$ is the probability that in each of
  the $n_j^k$ steps where label $j$ was selected the threshold $\alpha$ was
  randomly chosen to be \tmtextit{at least as large as} $t$. For $0 < t < 1$,
  we conclude
  \begin{eqnarray}
    \mathbbm{P} (c^k_j < t| (n_1^k, \ldots, n_l^k)) & = & \mathbbm{P} (c^k_j <
    t|n_j^k) \\
    & = & 1 -\mathbbm{P} (c_j^k \geqslant t|n_j^k) \\
    & \overset{0 < t < 1}{=} & 1 - \left( 1 - t \right)^{n_j^k} . 
    \label{eq:onefifteen}
  \end{eqnarray}
  The above formulation also covers the case $n_j^k = 0$ (note that we assumed
  $0 < t < 1$). For fixed $k$ the distributions of the $c_j^k$ are independent
  when conditioned on $(n_1^k, \ldots, n_l^k)$. Therefore we obtain from
  (\ref{eq:petc1}) and (\ref{eq:onefifteen})
  \begin{eqnarray}
    \mathbbm{P} (e^{\top} c < 1) & \overset{\text{(
    \ref{eq:petc1})}}{\geqslant} & \sum_{n_1^k + \ldots + n_l^k = k}
    \frac{k!}{n_1^k ! \cdot \ldots \cdot n_l^k !}  \left( \frac{1}{l}
    \right)^k \cdot\cont \\
    & & \quad\quad\prod_{j = 1}^l \mathbbm{P} \left( c_j^k < \frac{1}{l} |
    (n_1^k, \ldots, n_l^k) \right) \\
    & \overset{( \ref{eq:onefifteen})}{=} & \sum_{n_1^k + \ldots + n_l^k = k}
    \frac{k!}{n_1^k ! \cdot \ldots \cdot n_l^k !}  \left( \frac{1}{l}
    \right)^k  +\cont\\
    & & \qq \prod_{j = 1}^l \left( 1 - \left( 1 - \frac{1}{l}
    \right)^{n_j^k} \right) . 
  \end{eqnarray}
  Expanding the product and swapping the summation order, we derive
  \begin{eqnarray}
    & & \mathbbm{P} (e^{\top} c^k < 1) \\
    & \geqslant & \sum_{n_1^k + \ldots + n_l^k
    = k} \frac{k!}{n_1^k ! \cdot \ldots \cdot n_l^k !}  \left( \frac{1}{l}
    \right)^k  \mcont\\
    & & \qq\sum_{p \in \{0, 1\}^l} \prod_{j = 1}^l \left( - \left( 1 -
    \frac{1}{l} \right)^{n_j^k} \right)^{p_j}\\
    & = & \sum_{p \in \{0, 1\}^l} \left( - 1 \right)^{e^{\top} p} 
    \sum_{n_1^k + \ldots + n_l^k = k} \frac{k!}{n_1^k ! \cdot \ldots \cdot
    n_l^k !} \mcont\\
    & & \qq\prod_{j = 1}^l \left( \frac{1}{l}  \left( 1 - \frac{1}{l}
    \right)^{p_j} \right)^{n_j^k }\,.
  \end{eqnarray}
  Using the multinomial summation formula, we conclude
  \begin{eqnarray} 
    && \mathbbm{P} (e^{\top} c^k < 1) \geqslant\nn\\
    && \qq\sum_{p \in \{0, 1\}^l} \left( - 1
    \right)^{e^{\top} p}  \left( \underbrace{\sum_{j = 1}^l \frac{1}{l} 
    \left( 1 - \frac{1}{l} \right)^{p_j}}_{= : q_p} \right)^k, 
    \label{eq:laststareq}
  \end{eqnarray}
  which proves (\ref{eq:convprop1}). At $(\ast)$ the multinomial summation
  formula was invoked. Note that in (\ref{eq:laststareq}) the $n_j^k$ do not
  occur explicitly anymore. To show the second assertion (\ref{eq:convprop2}),
  we use the fact that, for any  $p \neq (0, \ldots, 0)$, $q_p$ can be bounded by $0 < q_p < 1$.
  Therefore
  \begin{eqnarray}
    \mathbbm{P} (e^{\top} c^k < 1) & \geqslant & q_0 + \sum_{p \in \{0, 1\}^l,
    p \neq 0} ( - 1)^{e^{\top} p}  (q_p)^k \\
    & = & 1 + \sum_{p \in \{0, 1\}^l, p \neq 0} ( - 1)^{e^{\top} p} 
    \underbrace{(q_p)^k}_{\overset{k \rightarrow \infty}{\rightarrow} 0} \\
    & \overset{k \rightarrow \infty}{\rightarrow} & 1, 
  \end{eqnarray}
  which proves (\ref{eq:convprop2}).
\end{proof}

We now show that Alg.~1 generates a sequence in
$\tmop{BV} (\Omega)^l$ almost surely. The {\tmem{perimeter}} of a set $A$ is
defined as the total variation of its characteristic function $\tmop{Per} (A)
\assign \tmop{TV} (1_A)$.

\begin{proposition}
  \label{prop:ukvalid}For the sequences $(u^k)$, $(U^k)$ generated by
  Alg.~1, define
  \begin{eqnarray}
    A & \assign & \bigcap_{k = 1}^{\infty} \{\gamma \in \Gamma | \tmop{Per}
    (U^k_{\gamma}) < \infty\} . 
  \end{eqnarray}
  Then
  \begin{eqnarray}
    \mathbbm{P} (A) & = & 1. 
  \end{eqnarray}
  If $\tmop{Per} (U^k_{\gamma}) < \infty$ for all $k$, then $u^k_{\gamma} \in
  \tmop{BV} (\Omega)^l$ for all $k$ as well. Moreover,
  \begin{eqnarray}
    \mathbbm{P} (u^k \in \tmop{BV} (\Omega)^l \wedge \tmop{Per} (U^k) < \infty
    \forall k \in \mathbbm{N}) & = & 1,  \label{eq:pukbvom}
  \end{eqnarray}
  i.e.,~the algorithm almost surely generates a sequence of $\tmop{BV}$
  functions $(u^k)$ and a sequence of sets of finite perimeter $(U^k)$.
\end{proposition}

\begin{proof}
  We first show that if $\tmop{Per} (U^{k'}) < \infty$ for all $k' \leqslant
  k$, then $u^k \in \tmop{BV} (\Omega)^l$ for all $k' \leqslant k$ as well.
  For $k = 0$, the assertion holds, since $u^0 = u \in \tmop{BV} (\Omega)^l$
  by assumption. For $k \geqslant 1$,
  \begin{eqnarray}
    u^k & = & e^{i^k} 1_{M^k} + u^{k - 1} 1_{\Omega \setminus M^k} . 
  \end{eqnarray}
  Since $M^k = U^{k - 1} \cap (\Omega \setminus U^k)$, and $U^k, U^{k - 1}$
  are assumed to have finite perimeter, $M^k$ also has finite perimeter.
  Applying {\cite[Thm.~3.84]{Ambrosio2000}} together with the boundedness of
  $u^{k - 1}$ and $u^{k - 1} \in \tmop{BV} (\Omega)^l$ by induction then
  provides $u^k \in \tmop{BV} (\Omega)^l$.
  
  We now denote
  \begin{eqnarray}
    I^k & \assign & \{\gamma \in \Gamma | \tmop{Per} (U^k_{\gamma}) =
    \infty\}, 
  \end{eqnarray}
  and the event that the {\tmem{first}} set with non-finite perimeter is
  encountered at step $k \in \mathbbm{N}_0$ by
  \begin{eqnarray}
    B^k & \assign & I^k \cap \left( \Gamma \setminus I^{k - 1} \right) \cap
    \ldots \cap \left( \Gamma \setminus I^0 \right) . 
  \end{eqnarray}
  Then
  \begin{eqnarray}
    \mathbbm{P} (A) & = & 1 -\mathbbm{P} \left( \bigcup_{k = 0}^{\infty} B^k
    \right) . 
  \end{eqnarray}
  As the sets $B^k$ are pairwise disjoint, and due to the countable additivity
  of the probability measure, we have
  \begin{eqnarray}
    \mathbbm{P} (A) & = & 1 - \sum_{k = 0}^{\infty} \mathbbm{P}(B^k) . 
    \label{eq:pvalidg}
  \end{eqnarray}
  Now $U^0 = \Omega$, therefore $\tmop{Per} (U^0) = \tmop{TV} (1_{U^0}) = 0 <
  \infty$ and $\mathbbm{P} (B^0) = 0$. For $k \geqslant 1$, we have
  \begin{eqnarray}
    \mathbbm{P} (B^k)
    & \leqslant & \mathbbm{P} \left( \tmop{Per} (U^k) =
    \infty \wedge \tmop{Per} (U^{k'}) < \infty \; \forall k' < k \right) \nn\\
    & \leqslant & \mathbbm{P} \left( \tmop{Per} (U^k) = \infty | \tmop{Per}
    (U^{k'}) < \infty \; \forall k' < k \right) \nn\\
    & = & \mathbbm{P} \big( \tmop{Per} (U^{k - 1} \cap \{^{} x \in \Omega
    |u_{i^k}^{k - 1} (x) \leqslant \alpha^k \}) = \infty | \cont\\
    & & \qq\tmop{Per} (U^{k'})
    < \infty \; \forall k' < k \big) .  \label{eq:pperukm1}
  \end{eqnarray}
  By the argument from the beginning of the proof, we know that $u^{k - 1} \in
  \tmop{BV} (\Omega)^l$ under the condition on the perimeter $\tmop{Per} (U^{k'})$,
  therefore from {\cite[Thm.~3.40]{Ambrosio2000}} we conclude that $\tmop{Per} (\{^{} x \in \Omega |u_{i^k}^{k - 1} (x) \leqslant
  \alpha^k \})$ is finite for $\mathcal{L}^1$-a.e.~$\alpha^k$ and all $i^k$.
  As the sets of finite perimeter are closed under finite intersection, and
  since the $\alpha^k$ are drawn from an uniform distribution, this implies that
  \begin{eqnarray}
    \mathbbm{P} (\tmop{Per} (U^k) < \infty | \tmop{Per} (U^{k - 1}) < \infty)
    & = & 1. 
  \end{eqnarray}
  Together with (\ref{eq:pperukm1}) we arrive at
  \begin{eqnarray}
    \mathbbm{P} (B^k) & = & 0. 
  \end{eqnarray}
  Substituting this result into (\ref{eq:pvalidg}) leads to the assertion,
  \begin{eqnarray}
    \mathbbm{P} (A) & = & 1. 
  \end{eqnarray}
  Equation (\ref{eq:pukbvom}) follows immediately.
\end{proof}

Using these propositions, we now formulate the main result of this section:
Alg.~1 almost surely generates an integral labeling
that is of bounded variation.

\begin{theorem}
  \label{thm:probterminates}Let $u \in \tmop{BV} (\Omega)^l$ and $f \left(
  \bar{u} \right)$ as in Def.~\ref{def:fubar}. Then
  \begin{eqnarray}
    \mathbbm{P} (f ( \bar{u}) < \infty) & = & 1. 
  \end{eqnarray}
\end{theorem}

\begin{proof}
  The first part is to show that $(u^k)$ becomes stationary almost surely,
  i.e.,
  \begin{eqnarray}
    \mathbbm{P} (\exists k \in \mathbbm{N} : U^k = \emptyset) & = & 1. 
    \label{eq:pgammaek}
  \end{eqnarray}
  Assume there exists $k$ such that $e^{\top} c^k < 1$, and assume further
  that $U^k \neq \emptyset$, i.e.,~there exists some $x \in U^k$. Then $u_j (x)
  \leqslant c^k_j$ for all labels $j$. But then $e^{\top} u (x) \leqslant
  e^{\top} c^k < 1$, which is a contradiction to $u (x) \in \Delta_l$.
  Therefore $U^k$ must be empty. From this observation and
  Prop.~\ref{prop:petckl1} we conclude, for all $k' \in \mathbbm{N}$,
  \begin{eqnarray}
    1 \geqslant \mathbbm{P} (\exists k \in \mathbbm{N}: U^k = \emptyset)
    \geqslant \text{} \mathbbm{P} (e^{\top} c^{k'} < 1) \overset{k'
    \rightarrow \infty}{\rightarrow} 1, &  & 
  \end{eqnarray}
  which proves (\ref{eq:pgammaek}).
  
  In order to show that $f ( \bar{u}_{\gamma}) < \infty$ with probability~$1$,
  it remains to show that the result is almost surely in $\tmop{BV}
  (\Omega)^l$. A sufficient condition is that almost surely {\tmem{all}}
  iterates~$u^k$ are elements of $\tmop{BV} (\Omega)^l$, i.e.,~
  \begin{eqnarray}
    \mathbbm{P} \left( u^k \in \tmop{BV} (\Omega)^l \hspace{0.75em} \forall k
    \in \mathbbm{N} \right) & = & 1.  \label{eq:pukbvom2}
  \end{eqnarray}
  This is shown by Prop.~\ref{prop:ukvalid}. Then
  \begin{eqnarray}
    & & \mathbbm{P} (f ( \bar{u}) < \infty) \\
    & \geqslant & \mathbbm{P} (\{\exists k \in \mathbbm{N}: U^k = \emptyset\} \wedge \{u^k \in \tmop{BV} (\Omega)^l \hspace{0.75em} \forall k \in \mathbbm{N} \}) \nn\\
    & = & \mathbbm{P} (\{u^k \in \tmop{BV} (\Omega)^l \hspace{0.75em} \forall
    k \in \mathbbm{N} \}) \\
    &  & -\mathbbm{P} (\{\forall k \in \mathbbm{N}: U^k \neq \emptyset\}    \wedge \{u^k \in \tmop{BV} (\Omega)^l \hspace{0.75em} \forall k \in    \mathbbm{N} \}) \nn\\
    & \overset{( \ref{eq:pukbvom2})}{=} & \mathbbm{P} (\{u^k \in \tmop{BV}
    (\Omega)^l \hspace{0.75em} \forall k \in \mathbbm{N} \}) - 0 \\
    & = & 1. 
  \end{eqnarray}
  Thus $\mathbbm{P} (f ( \bar{u}) < \infty) = 1$, which proves the
  assertion.{\qed}
\end{proof}

\section{Proof of the Main Theorem}\label{sec:proofmainthm}

In order to show the bound (\ref{eq:thegoal}) and Thm.~\ref{thm:mainthm}, we first need several technical
propositions regarding the composition of two $\tmop{BV}$ functions along a
set of finite perimeter. We denote by $(E)^1$ and $(E)^0$ the
measure-theoretic interior and exterior of a set $E$,
see~{\cite{Ambrosio2000}},
\begin{eqnarray}
  (E)^t & \assign & \{x \in \Omega | \lim_{\rho \searrow 0}
  \frac{|\mathcal{B}_{\rho} (x) \cap E|}{|\mathcal{B}_{\rho} (x) |} = t\},
  \hspace{1em} t \in [0, 1] . \label{eq:densityboundary}
\end{eqnarray}
Here $\mathcal{B}_{\rho} (x)$ denotes the ball with radius $\rho$ centered in
$x$, and $|A| \assign \mathcal{L}^d (A)$ the Lebesgue content of a set $A
\subseteq \mathbbm{R}^d$.

\begin{proposition}
  \label{prop:psimorebounds}Let $\Psi$ be positively homogeneous and convex,
  and satisfy the upper-boundedness condition (\ref{eq:psiupper}). Then
  \begin{eqnarray}
    \Psi (\nu (z^1 - z^2)^{\top}) & \leqslant & \lambda_u \hspace{1em} \forall
    z^1, z^2 \in \Delta_l .  \label{eq:psinubnd}
  \end{eqnarray}
  Moreover, there exists a constant $C < \infty$ such that
  \begin{eqnarray}
    & & \Psi (w) \leqslant C \|w\|_2 \hspace{1em} \forall w \in W\,,\\
    & & \quad W \assign \{w
    = (w^1 | \ldots |w^l) \in \mathbbm{R}^{d \times l} | \sum_{i = 1}^l w^i =
    0\} .  \label{eq:psiwbnd}
  \end{eqnarray}
\end{proposition}

\begin{proof}
See appendix.
\end{proof}

\begin{proposition}
  \label{prop:intcap}Let $E, F \subseteq \Omega^d$ be
  $\mathcal{L}^d$-measurable sets. Then
  \begin{eqnarray}
    (E \cap F)^1 & = & (E)^1 \cap (F)^1 . 
  \end{eqnarray}
\end{proposition}

\begin{proof}
See appendix.
\end{proof}

\begin{proposition}
  \label{prop:bvcomposed}Let $u, v \in \tmop{BV} (\Omega, \Delta_l)$ and $E
  \subseteq \Omega$ such that~$\tmop{Per} (E) < \infty$. Define
  \begin{eqnarray}
    w & \assign & u 1_E + v 1_{\Omega \setminus E} . 
  \end{eqnarray}
  Then $w \in \tmop{BV} (\Omega, \Delta_l)^l$, and
  \begin{eqnarray}
    D w & = & D u \llcorner (E)^1 + D v \llcorner (E)^0 +\cont\\&&\quad \nu_E  \left(
    u^+_{\mathcal{F}E} - v^-_{\mathcal{F}E} \right)^{\top} \mathcal{H}^{d - 1}
    \llcorner \left( \mathcal{F}E \cap \Omega \right),  \label{eq:dwedue}
  \end{eqnarray}
  where $u^+_{\mathcal{F}E}$ and $v^-_{\mathcal{F}E}$ denote the one-sided
  approximate limits of $u$ and~$v$ on the reduced boundary $\mathcal{F}E$,
  and $\nu_E$ is the generalized inner normal of $E$ {\cite{Ambrosio2000}}.
  Moreover, for continuous, convex and positively homogeneous $\Psi$
  satisfying the upper-boundedness condition (\ref{eq:psiupper}) and some
  Borel set $A \subseteq \Omega$,
  \begin{eqnarray}
    \int_A d \Psi (D w) & \leqslant & \int_{A \cap (E)^1} d \Psi (D u) +\cont\\
    & &\quad \int_{A \cap (E)^0} d \Psi (D v) + \lambda_u \tmop{Per} (E) . 
    \label{eq:intopsidw}
  \end{eqnarray}
\end{proposition}

\begin{proof}
See appendix.
\end{proof}

\begin{proposition}
  \label{prop:interiorsameobjective}Let $u, v \in \tmop{BV} (\Omega,
  \Delta_l)$, $E \subseteq \Omega$ such that $\tmop{Per} (E) < \infty$, and
  \begin{eqnarray}
    & u|_{(E)^1} = v|_{(E)^1} \hspace{1em} \mathcal{L}^d \text{-a.e.} & 
    \label{eq:assue1}
  \end{eqnarray}
  Then $(D u) \llcorner (E)^1 = (D v) \llcorner (E)^1$, and $\Psi (D u)
  \llcorner (E)^1 = \Psi (D v) \llcorner (E)^1$. In particular,
  \begin{eqnarray}
    \int_{(E)^1} d \Psi (D u) & = & \int_{(E)^1} d \Psi (D v) . 
    \label{eq:ie1psi}
  \end{eqnarray}
  The result also holds when $(E)^1$ is replaced by $(E)^0$. Moreover, the
  condition (\ref{eq:assue1}) is equivalent to
  \begin{eqnarray}
    & u|_E = v|_E \hspace{1em} \mathcal{L}^d \text{- a.e.} & 
    \label{eq:ureeqvre}
  \end{eqnarray}
\end{proposition}

\begin{proof}
See appendix.
\end{proof}

\begin{remark}
  Note that taking the measure-theoretic interior $(E)^1$ is of central
  importance. The corollary does not hold when replacing the integral over
  $(E)^1$ with the integral over~$E$, as can be seen from the example of the
  closed unit ball, i.e.,~$E =\mathcal{B}_1 (0)$, $u = 1_E$ and $v \equiv 1$.
\end{remark}

\subsection{Proof of Theorem~\ref{thm:mainthm}}

In Sect.~\ref{sec:probroundterm} we have shown that the rounding process
induced by Alg.~1 is well-defined in the sense that
it returns an integral solution $\bar{u}_{\gamma} \in \tmop{BV} (\Omega)^l$
almost surely. We now return to proving an upper bound for the expectation of
$f ( \bar{u})$ as in the approximate coarea
formula~(\ref{eq:probrounddesired}). We first show that the expectation of
the {\tmem{linear part}} (data term) of $f$ is invariant under the rounding
process.

\begin{proposition}
  \label{prop:euks}The sequence $(u^k)$ generated by
  Alg.~1 satisfies
  \begin{eqnarray}
    \mathbbm{E} (\langle u^k, s \rangle) & = & \langle u, s \rangle
    \hspace{1em} \forall k \in \mathbbm{N} . 
  \end{eqnarray}
\end{proposition}

\begin{proof}
  In Alg.~1, instead of step \ref{step:assignu} we
  consider the simpler update
  \begin{eqnarray}
    u^k & \leftarrow & e^{i^k} 1_{\{u_{i^k}^{k - 1} > \alpha^k \}} + u^{k - 1}
    1_{\{u_{i^k}^{k - 1} \leqslant \alpha^k \}} .  \label{eq:assignstepsimple}
  \end{eqnarray}
  This yields exactly the same sequence $(u^k)$, since if $u_{i^k}^{k - 1} (x) >
  \alpha^k$ for any $\alpha^k \geqslant 0$, then either $x \in U^{k -
  1}$, or $u_{i^k}^{k - 1} (x) = 1$. In both algorithms, points that are
  assigned a label $e^{i^k}$ at some point in the process will never be
  assigned a {\tmem{different}} label at a later point. This is made explicit
  in Alg.~1 by keeping track of the set $U^k$ of yet
  unassigned points. In contrast, using the step (\ref{eq:assignstepsimple}),
  a point may formally be assigned the same label multiple times.
  
  Denote $\gamma' \assign (\gamma^1, \ldots, \gamma^{k - 1})$ and $u^{\gamma'}
  \assign u_{\gamma}^{k - 1}$. We apply induction on $k$: For $k \geqslant 1$,
  \begin{eqnarray}
    & & \mathbbm{E}_{\gamma} \langle u_{\gamma}^k, s \rangle\\
    & = & \mathbbm{E}_{\gamma'} \frac{1}{l} \sum_{i = 1}^l \int_0^1 \sum_{j = 1}^l
    s_j \cdot \left( e^i 1_{\{u_i^{\gamma'} > \alpha\}} + u^{\gamma'}
    1_{\{u_i^{\gamma'} \leqslant \alpha\}} \right)_j d \alpha \nn\\
    & = & \mathbbm{E}_{\gamma'} \frac{1}{l} \sum_{i = 1}^l \int_0^1 \left(
    s_i \cdot 1_{\{u^{\gamma'}_i > \alpha\}} + u^{\gamma'} 1_{\{u_i^{\gamma'}
    \leqslant \alpha\}} \langle u^{\gamma'}, s \rangle \right) d \alpha \nn\\
    & = & \mathbbm{E}_{\gamma'} \frac{1}{l} \sum_{i = 1}^l \int_0^1 \Big(
    s_i \cdot 1_{\{u^{\gamma'}_i > \alpha\}} + \cont\\
    & & \qqqq\qqq\left( 1 - 1_{\{u_i^{\gamma'} >
    \alpha\}} \Big) \langle u^{\gamma'}, s \rangle \right) d \alpha\,. 
  \end{eqnarray}
  Now we take into account the property {\cite[Prop. 1.78]{Ambrosio2000}},
  which is a direct consequence of Fubini's theorem, and also used in the
  proof of the thresholding theorem for the two-class case \cite{Nikolova2006}:
  \begin{eqnarray}
    && \int_0^1 \int_{\Omega} s_i (x) \cdot 1_{\{u_i > \alpha\}} (x) d x d \alpha\\
    & = & \int_{\Omega} s_i (x) u_i (x) dx = \langle u_i, s_i \rangle . 
  \end{eqnarray}
  This leads to
  \begin{eqnarray}
    & & \mathbbm{E}_{\gamma} \langle u_{\gamma}^k, s \rangle\nn\\
    & = & \mathbbm{E}_{\gamma'} \frac{1}{l} \sum_{i = 1}^l \left( s_i u_i^{\gamma'}
    + \langle u^{\gamma'}, s \rangle - u_i^{\gamma'} \langle u^{\gamma'}, s
    \rangle \right) d \alpha
  \end{eqnarray}
  and therefore, using $u^{\gamma'} (x) \in \Delta_l$,
  \begin{eqnarray}
    \mathbbm{E}_{\gamma} \langle u_{\gamma}^k, s \rangle = \mathbbm{E}_{\gamma'}
    \langle u^{\gamma'}, s \rangle =
    \mathbbm{E}_{\gamma} \langle u_{\gamma}^{k - 1}, s \rangle . 
  \end{eqnarray}
  Since $\langle u^0, s \rangle = \langle u, s \rangle$, the assertion follows
  by induction.{\qed}
\end{proof}

\begin{remark}
  Prop.~\ref{prop:euks} shows that the data term is -- in the mean -- not
  affected by the probabilistic rounding process, i.e.,~it satisfies an
  {\tmem{exact}} coarea-like formula, even in the multiclass case.
\end{remark}

Bounding the regularizer is more involved: For $\gamma^k = (i^k, \alpha^k)$,
define
\begin{eqnarray}
  U_{\gamma^k} & \assign & \{x \in \Omega |u_{i^k} (x) \leqslant \alpha^k \},
  \\
  V_{\gamma^k} & \assign & \left( U_{\gamma^k} \right)^1, \\
  V^k & \assign & (U^k)^1 . 
\end{eqnarray}
As the measure-theoretic interior is invariant under
$\mathcal{L}^d$-negligible modifications, given some fixed sequence $\gamma$
the sequence $(V^k)$ is invariant under $\mathcal{L}^d$-negligible
modifications of $u = u^0$, i.e.,~it is uniquely defined when viewing $u$ as an
element of $L^1 (\Omega)^l$. Some calculations yield
\begin{eqnarray}
  U^k & = & U_{\gamma^1} \cap \ldots \cap U_{\gamma^k}, \hspace{1em} k
  \geqslant 1, \\
  U^{k - 1} \setminus U^k & = & U_{\gamma^1} \cap \Big( \left( U_{\gamma^2}
  \cap \ldots \cap U_{\gamma^{k - 1}} \right) \setminus\cont\\
  && \qq\quad\left( U_{\gamma^2}
  \cap \ldots \cap U_{\gamma^k} \right) \Big), \hspace{1em} k \geqslant 2. 
\end{eqnarray}
From these observations and Prop.~\ref{prop:intcap},
\begin{eqnarray}
  V^k & = & V_{\gamma^1} \cap \ldots \cap V_{\gamma^k}, \hspace{1em} k
  \geqslant 1, \\
  V^{k - 1} \setminus V^k & = & V_{\gamma^1} \cap \Big( \left( V_{\gamma^2}
  \cap \ldots \cap V_{\gamma^{k - 1}} \right) \setminus\cont\\
  && \qq\quad\left( V_{\gamma^2}
  \cap \ldots \cap V_{\gamma^k} \right) \Big), \hspace{1em} k \geqslant 2, 
  \label{eq:vmk1vkc}\\
  \Omega \setminus V^k & = & \bigcup_{k' = 1}^k \left( V^{k' - 1} \setminus
  V^{k'} \right), \hspace{1em} k \geqslant 1.  \label{eq:omvkunion}
\end{eqnarray}
Moreover, since $V^k$ is the measure-theoretic interior of $U^k$, both sets
are equal up to an $\mathcal{L}^d$-negligible set (cf.~(\ref{eq:e1eqe})).

We now prepare for an induction argument on the expectation of the
regularizing term when restricted to the sets $V^{k - 1} \setminus V^k$. The
following proposition provides the initial step ($k = 1$).

\begin{proposition}
  \label{prop:v0v1initial}Assume that $\Psi$ satisfies the lower- and upper-boundedness
  conditions (\ref{eq:psilower}) and (\ref{eq:psiupper}). Then
  \begin{eqnarray}
    \mathbbm{E} \int_{V^0 \setminus V^1} d \Psi (D \bar{u}) & \leqslant &
    \frac{2}{l} \frac{\lambda_u}{\lambda_l} \int_{\Omega} d \Psi (Du) . 
  \end{eqnarray}
\end{proposition}

\begin{proof}
  Denote $(i, \alpha) = \gamma^1$. Since $1_{U_{(i, \alpha)}} = 1_{V_{(i,
  \alpha)}}$ $\mathcal{L}^d$-a.e., we have
  \begin{eqnarray}
    & \bar{u}_{\gamma} = & 1_{V_{(i, \alpha)}} e^i + 1_{\Omega \setminus
    V_{(i, \alpha)}} \bar{u}_{\gamma} \hspace{1em} \mathcal{L}^d \text{- a.e.}    
  \end{eqnarray}
  Therefore, since $V^0 = (U^0)^1 = (\Omega)^1 = \Omega$,
  \begin{eqnarray}
    &&\int_{V^0 \setminus V^1} d \Psi (D \bar{u}_{\gamma}) = \int_{\Omega
    \setminus V_{(i, \alpha)}^{}} d \Psi (D \bar{u}_{\gamma}) \nn\\
    &=& \int_{\Omega
    \setminus V_{(i, \alpha)}^{}} d \Psi \left( D \left( 1_{V_{(i, \alpha)}}
    e^i + 1_{\Omega \setminus V_{(i, \alpha)}} \bar{u}_{\gamma} \right)
    \right) .
  \end{eqnarray}
  Since $u \in \tmop{BV} (\Omega)^l$, we know that $\tmop{Per} (V_{(i,
  \alpha)}) < \infty$ holds for $\mathcal{L}^1$-a.e. $\alpha$ and any $i$
  {\cite[Thm.~3.40]{Ambrosio2000}}. Therefore we conclude from
  Prop.~\ref{prop:bvcomposed} that for $\mathcal{L}^1$-a.e.~$\alpha$,
  \begin{eqnarray}
    &  & \int_{\Omega \setminus V_{(i, \alpha)}^{}} d \Psi (D
    \bar{u}_{\gamma}) \leqslant \lambda_u \tmop{Per} \left( V_{(i, \alpha)}
    \right) + \cont\\
    &  & \qqq \int_{\left( \Omega \setminus V_{(i, \alpha)}^{}
    \right) \cap \left( \Omega \setminus V_{(i, \alpha)}^{} \right)^1} d \Psi
    \left( De^i \right) + \cont\\
    & & \qqq \int_{\left( \Omega \setminus V_{(i, \alpha)}^{}
    \right) \cap \left( \Omega \setminus V_{(i, \alpha)}^{} \right)^0} d \Psi
    \left( D \bar{u}_{\gamma} \right) \hspace{0.25em} . 
  \end{eqnarray}
  Both of the integrals are zero, since $De^i = 0$ and
  \begin{eqnarray}
  (\Omega \setminus V_{(i, \alpha)})^0 = (V_{(i, \alpha)})^1 = V_{(i, \alpha)}\,,
  \end{eqnarray}
  therefore
  \begin{eqnarray}
  \int_{\Omega \setminus V_{(i, \alpha)}^{}} d \Psi (D \bar{u}_{\gamma}) \leqslant \lambda_u \tmop{Per} (V_{(i, \alpha)})\,.
  \end{eqnarray}
  Carrying the bound over
  to the expectation yields
  \begin{eqnarray}
    \mathbbm{E}_{\gamma} \int_{\Omega \setminus V_{(i, \alpha)}^{}} d \Psi (D
    \bar{u}_{\gamma}) & \leqslant & \frac{1}{l} \sum_{i = 1}^l \int_0^1
    \lambda_u \tmop{Per} (V_{(i, \alpha)}) d \alpha\,.\nn
  \end{eqnarray}
  Also, $\tmop{Per} (V_{(i, \alpha)}) = \tmop{Per} (U_{(i, \alpha)})$ since
  the perimeter is invariant under $\mathcal{L}^d$-negligible modifications.
  The assertion then follows using $V^0 = \Omega$, $V^1 = V_{(i, \alpha)}$ and
  the coarea formula:
  \begin{eqnarray}
    & & \mathbbm{E}_{\gamma} \int_{V^0 \setminus V^1} d \Psi (D \bar{u}_{\gamma})\\
    & \leqslant & \frac{1}{l} \sum_{i = 1}^l \int_0^1 \lambda_u \tmop{Per}
    (U_{(i, \alpha)}) d \alpha \\
    & \overset{\tmop{coarea}}{=} & \frac{\lambda_u}{l} \sum_{i = 1}^l
    \tmop{TV} (u_i) 
     =  \frac{\lambda_u}{l} \int_{\Omega} \sum_{i = 1}^l d\|D u_i \|_2\\
    & \overset{\text{(\ref{eq:psilower})}}{\leqslant} & \frac{2}{l}
    \frac{\lambda_u}{\lambda_l} \int_{\Omega} d \Psi (Du) . 
  \end{eqnarray}
\end{proof}

We now take care of the induction step for the regularizer bound.

\begin{proposition}
  \label{prop:egamlm1l}Let $\Psi$ satisfy the upper-boundedness condition
  (\ref{eq:psiupper}). Then, for any $k \geqslant 2$,
  \begin{eqnarray}
    F &\assign& \mathbbm{E} \int_{V^{k - 1} \setminus V^k} d \Psi (D \bar{u})\\
    &\leqslant & \frac{(l - 1)}{l}  \mathbbm{E} \int_{V^{k - 2} \setminus V^{k - 1}} d \Psi (D \bar{u}) . 
  \end{eqnarray}
\end{proposition}

\begin{proof}
  Define the shifted sequence $\gamma' = (\gamma'^k)_{k = 1}^{\infty}$ by
  $\gamma'^k \assign \gamma^{k + 1}$, and let
  \begin{eqnarray}
    W_{\gamma'} & \assign & V_{\gamma'}^{k - 2} \setminus V_{\gamma'}^{k - 1}\\
    &=& \left( V_{\gamma^2} \cap \ldots \cap V_{\gamma^{k - 1}} \right)
    \setminus \left( V_{\gamma^2} \cap \ldots \cap V_{\gamma^k} \right) . 
    \label{eq:vkgprime}
  \end{eqnarray}
  By Prop.~\ref{thm:probterminates} we may assume that, under the expectation,
  $\bar{u}_{\gamma}$ exists and is an element of~$\tmop{BV} (\Omega)^l$. We
  denote $\gamma^1 = (i, \alpha)$, then $V^{k - 1} \setminus V^k = V_{(i,
  \alpha)} \cap W_{\gamma'}$ due to (\ref{eq:vmk1vkc}). For each pair~$(i,
  \alpha)$ we denote by $((i, \alpha), \gamma')$ the sequence obtained by
  prepending $(i, \alpha)$ to the sequence $\gamma'$. Then
  \begin{eqnarray}
    F & = & \frac{1}{l} \sum_{i = 1}^l \int_0^1 \mathbbm{E}_{\gamma'}
    \int_{V_{(i, \alpha)} \cap W_{\gamma'}} d \Psi (D \bar{u}_{((i, \alpha),
    \gamma')}) d \alpha .  \label{eq:onebylsumi}
  \end{eqnarray}
  Since in the first iteration of the algorithm no points in $U_{(i, \alpha)}$
  are assigned a label, $\bar{u}_{((i, \alpha), \gamma')} = \bar{u}_{\gamma'}$
  holds on $U_{(i, \alpha)}$, and therefore $\mathcal{L}^d$-a.e.~on $V_{(i,
  \alpha)}$. Therefore we may apply Prop.~\ref{prop:interiorsameobjective} and
  substitute $D \bar{u}_{((i, \alpha), \gamma')}$ by $D \bar{u}_{\gamma'}$ in
  (\ref{eq:onebylsumi}):
  \begin{eqnarray}
    F & = & \frac{1}{l} \sum_{i = 1}^l \int_0^1 \left( \mathbbm{E}_{\gamma'}
    \int_{V_{(i, \alpha)} \cap W_{\gamma'}} d \Psi (D \bar{u}_{\gamma'})
    \right) \hspace{-0.25em} d \alpha \\
    & = & \frac{1}{l} \sum_{i = 1}^l \int_0^1 \left( \mathbbm{E}_{\gamma'}
    \int_{W_{\gamma'}} 1_{V_{(i, \alpha)}} d \Psi (D \bar{u}_{\gamma'})
    \right) \hspace{-0.25em} d \alpha . 
  \end{eqnarray}
  By definition of the measure-theoretic interior
  (\ref{eq:densityboundary}), the indicator function $1_{V_{(i, \alpha)}}$ is bounded from
  above by the density function $\Theta_{U_{(i, \alpha)}}$ of $U_{(i,
  \alpha)}$,
  \begin{eqnarray}
    1_{V_{(i, \alpha)}} (x) & \leqslant & \Theta_{(i, \alpha)} (x) \assign_{}
    \lim_{\delta \searrow 0} \frac{|\mathcal{B}_{\delta} (x) \cap U_{(i,
    \alpha)} |}{|\mathcal{B}_{\delta} (x) |}, 
  \end{eqnarray}
  which exists $\mathcal{H}^{d - 1}$-a.e.~on $\Omega$ by {\cite[Prop.
  3.61]{Ambrosio2000}}. Therefore, denoting by
  $\mathcal{B}_{\delta} (\cdummy)$ the mapping $x \in \Omega \mapsto
  \mathcal{B}_{\delta} (x)$,
  \begin{eqnarray}
    F & \leqslant & \frac{1}{l} \sum_{i = 1}^l \int_0^1 
    \mathbbm{E}_{\gamma'} \int_{W_{\gamma'}}  \lim_{\delta \searrow 0}
    \frac{| \mathcal{B}_{\delta} (\cdot) \cap U_{(i, \alpha)} |}{|
    \mathcal{B}_{\delta} (\cdot) |}  d \Psi (D \bar{u}_{\gamma'})
 d \alpha . \nn
  \end{eqnarray}
  Rearranging the integrals and the limit, which can be justified by $\tmop{TV} ( \bar{u}_{\gamma'}) < \infty$ almost surely
  and dominated convergence using (\ref{eq:psiupper}), we get
  \begin{eqnarray}
    F & \leqslant & \frac{1}{l}  \mathbbm{E}_{\gamma'} \lim_{\delta \searrow
    0} \int_{W_{\gamma'}} \sum_{i = 1}^l \int_0^1 \frac{|
    \mathcal{B}_{\delta} (\cdot) \cap U_{(i, \alpha)} |}{|
    \mathcal{B}_{\delta} (\cdot) |} d \alpha\  d \Psi (D
    \bar{u}_{\gamma'}) \nn\\
    & = & \frac{1}{l}  \mathbbm{E}_{\gamma'} \lim_{\delta \searrow 0}
    \int_{W_{\gamma'}} \frac{1}{| \mathcal{B}_{\delta} (\cdot) |} \cdot\\
    & & \qq\left(
    \sum_{i = 1}^l \int_0^1 \int_{\mathcal{B}_{\delta} (\cdummy)} 1_{\{u_i (y)
    \leqslant \alpha\}} dyd \alpha \right) d \Psi (D \bar{u}_{\gamma'})\,.\nn
  \end{eqnarray}
  We again apply {\cite[Prop. 1.78]{Ambrosio2000}} to the two innermost
  integrals (alternatively, use Fubini's theorem), which leads to
  \begin{eqnarray}
    F & \leqslant & \frac{1}{l}  \mathbbm{E}_{\gamma'} \lim_{\delta \searrow
    0} \int_{W_{\gamma'}} \frac{1}{| \mathcal{B}_{\delta} (\cdot) |} \cdot\\
    &&\qq\left( \sum_{i = 1}^l \int_{\mathcal{B}_{\delta} (\cdummy)} (1 - u_i (y)) dy
    \right) d \Psi (D \bar{u}_{\gamma'}) \hspace{0.25em} . 
  \end{eqnarray}
  Using the fact that $u (y) \in \Delta_l$, this collapses according to
  \begin{eqnarray}
    F & \leqslant & \frac{1}{l}  \mathbbm{E}_{\gamma'} \lim_{\delta \searrow
    0} \int_{W_{\gamma'}} \frac{1}{| \mathcal{B}_{\delta} (\cdot) |} \left(
    \int_{\mathcal{B}_{\delta} (\cdummy)} (l - 1) dy \right) d \Psi (D
    \bar{u}_{\gamma'}) \nn\\
    & = & \frac{1}{l}  \mathbbm{E}_{\gamma'} \lim_{\delta \searrow 0}
    \int_{W_{\gamma'}} (l - 1) d \Psi (D \bar{u}_{\gamma'}) \\
    & = & \frac{l - 1}{l}  \mathbbm{E}_{\gamma'} \int_{W_{\gamma'}} d \Psi (D
    \bar{u}_{\gamma'}) \\
    & = & \frac{l - 1}{l}  \mathbbm{E}_{\gamma'} \int_{V_{\gamma'}^{k - 2}
    \setminus V_{\gamma'}^{k - 1}} d \Psi (D \bar{u}_{\gamma'}) . 
  \end{eqnarray}
  Reverting the index shift and using $\bar{u}_{\gamma'} =
  \bar{u}_{\gamma}$ concludes the proof:
  \begin{eqnarray}
    F & \leqslant & \frac{l - 1}{l}  \mathbbm{E}_{\gamma} \int_{V_{\gamma}^{k
    - 1} \setminus V_{\gamma}^k} d \Psi (D \bar{u}_{\gamma}) \hspace{0.25em} .
    \label{eq:ffinalresult}
  \end{eqnarray}
\end{proof}

We are now ready to prove the main result, Thm.~\ref{thm:mainthm}, as stated
in the introduction.

\begin{proof}
  \textbf{(Theorem~\ref{thm:mainthm})} The fact that the algorithm
  provides $\bar{u} \in \mathcal{C}_{\mathcal{E}}$ almost surely follows from
  Thm.~\ref{thm:probterminates}. Therefore there almost surely exists $k'
  \assign k' (\gamma) \geqslant 1$ such that~$U^{k'} = \emptyset$ and
  $\bar{u}_{\gamma} = u_{\gamma}^{k'}$. On one hand, this implies
  \begin{eqnarray}
    & \langle \bar{u}_{\gamma}, s \rangle = \langle u_{\gamma}^{k'}, s
    \rangle = \lim_{k \rightarrow \infty} \langle u_{\gamma}^k, s \rangle & 
    \label{eq:someeqx1}
  \end{eqnarray}
  almost surely. On the other hand, $V^{k'} = (U^{k'})^1 = \emptyset$
  and therefore
  \begin{eqnarray}
    \bigcup_{k = 1}^{k'} V^{k - 1} \setminus V^k & \overset{(\ast)}{=} &
    \Omega \setminus V^{k'} = \Omega  \label{eq:someeqx2}
  \end{eqnarray}
  almost surely. The equality $(\ast)$ can be shown by induction: For the base
  case $k' = 1$, we have $V^0 = (U^0)^1 = (\Omega)^1 = \Omega$, since $\Omega$
  is the open unit box. For $k' \geqslant 2$,
  \begin{eqnarray}
    &&\bigcup_{k = 1}^{k'} V^{k - 1} \setminus V^k\\ & = & \left( V^{k' - 1}
    \setminus V^{k'} \right) \cup \bigcup_{k = 1}^{k' - 1} \left( V^{k - 1}
    \setminus V^k \right) \\
    & = & \left( V^{k' - 1} \setminus V^{k'} \right) \cup \left( \Omega
    \setminus V^{k' - 1} \right) \\
    & \overset{V^{k' - 1} \subseteq \Omega}{=} & \Omega \setminus V^{k' - 1}
    . 
  \end{eqnarray}
  almost surely (cf. (\ref{eq:omvkunion})). From (\ref{eq:someeqx1}) and
  (\ref{eq:someeqx2}) we obtain
  \begin{eqnarray}
    \mathbbm{E}_{\gamma} f ( \bar{u}_{\gamma}) & = & \mathbbm{E}_{\gamma}
    \left( \lim_{k \rightarrow \infty} \langle u_{\gamma}^k, s \rangle \right)+\cont\\
    &&\qq \mathbbm{E}_{\gamma} \left( \sum_{k = 1}^{\infty} \int_{V^{k - 1}
    \setminus V^k} d \Psi (D \bar{u}_{\gamma}) \right) \\
    & = & \lim_{k \rightarrow \infty} \left( \mathbbm{E}_{\gamma} \langle
    u_{\gamma}^k, s \rangle \right) +\cont\\
    &&\qq \sum_{k = 1}^{\infty}
    \mathbbm{E}_{\gamma} \int_{V^{k - 1} \setminus V^k} d \Psi (D
    \bar{u}_{\gamma})  \label{eq:swapswap}
  \end{eqnarray}
  The first term in (\ref{eq:swapswap}) is equal to $\langle u, s \rangle$ due
  to Prop.~\ref{prop:euks}. An induction argument using
  Prop.~\ref{prop:v0v1initial} and Prop.~\ref{prop:egamlm1l} shows that the
  second term can be bounded according to
  \begin{eqnarray}
    &&\sum_{k = 1}^{\infty} \mathbbm{E}_{\gamma} \int_{V^{k - 1} \setminus V^k}
    d \Psi (D \bar{u}_{\gamma}) \\ & \leqslant & \sum_{k = 1}^{\infty} \left(
    \frac{l - 1}{l} \right)^{k - 1} \frac{2}{l}  \frac{\lambda_u}{\lambda_l}
    \int_{\Omega} d \Psi (Du) \\
    & = & 2 \frac{\lambda_u}{\lambda_l} \int_{\Omega} d \Psi (Du)
    \hspace{0.25em}, 
  \end{eqnarray}
  therefore
  \begin{eqnarray}
    \mathbbm{E}_{\gamma} f ( \bar{u}_{\gamma}) & \leqslant & \langle u, s
    \rangle + 2 \frac{\lambda_u}{\lambda_l} \int_{\Omega} d \Psi (Du)
    \hspace{0.25em} . 
  \end{eqnarray}
  Since $s \geqslant 0$ and $\lambda_u \geqslant \lambda_l$, and
  therefore the linear term is bounded by $\langle u, s \rangle \leqslant 2 (\lambda_u / \lambda_l) \langle
  u, s \rangle$, this proves the assertion. Swapping the integral and limit
  in (\ref{eq:swapswap}) can be justified retrospectively by the dominated
  convergence theorem, using $0 \leqslant \langle u, s \rangle \leqslant
  \infty$ and $\int_{\Omega} d \Psi (Du) < \infty$ due to the
  upper-boundedness condition and
  Prop.~\ref{prop:psimorebounds}.{\qed}
\end{proof}

Corollary~\ref{cor:maincor} (see introduction) follows immediately using $f
(u^{\ast}) \leqslant f (u_{\mathcal{E}}^{\ast})$,
cf.~(\ref{eq:proboptimalityfromcoarea}). We have demonstrated that the
proposed approach allows to recover, from the solution $u^{\ast}$ of the
convex {\tmem{relaxed}} problem (\ref{eq:problemrelaxed}), an approximate
\tmtextit{integral} solution $\bar{u}^{\ast}$ of the nonconvex
{\tmem{original}} problem (\ref{eq:combprob}) with an upper bound on the
objective.

For the specific case $\Psi = \Psi_d$, we have

\begin{proposition}\label{prop:dijprop}
  Let $d : \mathcal{I}^2 \rightarrow \mathbbm{R}_{\geqslant 0}$ be a metric
  and $\Psi = \Psi_d$. Then one may set
  \begin{eqnarray*}
    \lambda_u = \max_{i, j \in \{1, \ldots, l\}} d (i, j) & \text{and} &
    \lambda_l = \min_{i \neq j} d (i, j) .
  \end{eqnarray*}
\end{proposition}

\begin{proof}
  From the remarks in the introduction we obtain (cf.~\cite{Lellmann2011c})
  \begin{eqnarray*}
    \Psi_d (\nu (e^i - e^j)) & = & d (i, j),
  \end{eqnarray*}
  which shows the upper bound. For the lower bound, set $c \assign \min_{i
  \neq j} d (i, j)$, $v'^i \assign \frac{c}{2}  \frac{w^i}{\|w^i \|_2}$ and $v
  \assign v'  (I - \frac{1}{l} ee^{\top}) $. Then $v \in \mathcal{D}_{\tmop{loc}}^d$, since $\|v^i - v^j \|_2 = \|v'^i -
  v'^j \|_2 \leqslant c$ and
  $ve = v'  (I - \frac{1^{}}{l} ee^{\top}) e = 0$.
  Therefore, for $w \in
  \mathbbm{R}^{d \times l}$ satisfying $we = 0$,
  \begin{eqnarray}
    \Psi_d (w) & \geqslant & \langle w, v \rangle = \langle w, v' \rangle\\
    & =&
    \sum_{i = 1}^l \langle w^i, \frac{c}{2}  \frac{w^i}{\|w^i \|_2} \rangle =
    \frac{c}{2}  \sum_{i = 1}^l \|w^i \|_2, 
  \end{eqnarray}
  proving the lower bound.
\end{proof}

Finally, for $\Psi_d$ we obtain the factor
\begin{eqnarray}
  &  & 2 \frac{\lambda_u}{\lambda_l} \hspace{1em} = \hspace{1em} 2
  \frac{\max_{i, j} d (i, j)}{\min_{i \neq j} d (i, j)}, 
\end{eqnarray}
determining the optimality bound, as claimed in the introduction (\ref{eq:psidclaim}). The bound in
(\ref{eq:efubarstar}) is the same as the known bounds for finite-dimensional
metric labeling {\cite{Kleinberg1999}} and $\alpha$-expansion
{\cite{Boykov2001}}, however it extends these results to problems on
continuous domains for a broad class of regularizers.

\section{Conclusion}
In this work we considered a method for recovering approximate
solutions of image partitioning problems from solutions of a convex relaxation.
We proposed a probabilistic rounding method motivated by the finite-dimensional framework,
and showed that it is possible to obtain \emph{a priori} bounds
on the optimality of the integral solution obtained by rounding a solution of the convex relaxation.

The obtained bounds are compatible with known bounds for the finite-dimensional setting.
However, to our knowledge, this is the first fully convex approach that is both formulated in the spatially continuous setting and provides a true \tmtextit{a priori} bound. We showed that the approach can also be interpreted as an
approximate variant of the coarea formula.

While the results apply to a quite general class of regularizers, they are formulated for the homogeneous case. Non-homogeneous regularizers constitute an interesting direction for future work. In particular, such regularizers naturally occur when 
applying convex relaxation techniques {\cite{Alberti2003,Pock2010}} in order to solve nonconvex variational problems.

With the increasing computational power, such techniques have become quite popular recently. For problems where the convexity is confined to the data term, they permit to find a global minimizer. A proper extension of the results outlined in this work may provide a way to find good approximate solutions of problems where also the \emph{regularizer} is nonconvex.

\section{Appendix}

\begin{proof}[Prop.~\ref{prop:psimorebounds}]
  In order to prove the first assertion (\ref{eq:psinubnd}), note that the
  mapping $w \mapsto \Psi (\nu w^{\top})$ is convex, therefore it must assume
  its maximum on the polytope $\Delta_l - \Delta_l \assign \{z^1 - z^2 |z^1,
  z^2 \Delta_l \}$ in a vertex of the polytope.. Since the polytope $\Delta_l - \Delta_l$ is
  the difference of two polytopes, its vertex set is at most the difference of
  their vertex sets, $V \assign \{e^i - e^j |i, j \in \{1, \ldots, l\}\}$. On
  this set, the bound $\Psi (\nu w^{\top}) \leqslant \lambda_u$ holds for $w \in V$ due to the upper-boundedness condition (\ref{eq:psiupper}), which proves~(\ref{eq:psinubnd}).
  
  The second equality (\ref{eq:psiwbnd}) follows from the fact that $G \assign
  \{b^{i k} \assign e^k (e^i - e^{i + 1})^{\top} |1\leqslant k \leqslant d, 1\leqslant i\leqslant l-1\}$ is a basis of the linear subspace $W$,
  satisfying $\Psi (b^{i k}) \leqslant \lambda_u$, and $\Psi$ is positively
  homogeneous and convex, and thus subadditive. Specifically, there is a
  linear transform $T : W \rightarrow \mathbbm{R}^{d \times (l - 1)}$ such
  that $w = \sum_{i, k} b^{i k} \alpha_{i k}$ for $\alpha = T (w)$. Then
  \begin{eqnarray}
    \Psi (w) & = & \Psi \left( \sum_{i, k} b^{i k} \alpha_{i k}\right) \\
    & \leqslant & \Psi \left( \sum_{i k} | \alpha_{i k} | \tmop{sgn} (\alpha_{i k}) b^{i k}\right) \\
    & \leqslant & \sum_{i k} | \alpha_{i k} | \Psi \left(\tmop{sgn} (\alpha_{i k}) b^{i k}\right)\,.
  \end{eqnarray}
  Since
  (\ref{eq:psiupper}) provides $\Psi (\pm b^{i k}) \leqslant \lambda_u$, we
  obtain
  \begin{equation}
  \Psi (w) \leqslant \lambda_u  \sum_{i k} | \alpha_{i k} | \leqslant
  \lambda_u  \|T\|  \|w\|_2
  \end{equation}
  for some suitable operator norm $\| \cdot \|$ and
  any $w \in W$.
\end{proof}

\begin{proof}[Prop.~\ref{prop:intcap}]
  We prove mutual inclusion:

    $'' \subseteq''$: From the definition of the measure-theoretic
    interior,
    \begin{eqnarray}
      x \in (E \cap F)^1 & \Rightarrow & \lim_{\delta \searrow 0}
      \frac{|\mathcal{B}_{\delta} (x) \cap E \cap F|}{|\mathcal{B}_{\delta}
      (x) |} = 1. 
    \end{eqnarray}
    Since $|\mathcal{B}_{\delta} (x) | \geqslant |\mathcal{B}_{\delta} (x)
    \cap E| \geqslant |\mathcal{B}_{\delta} (x) \cap E \cap F|$ (and vice
    versa for $|\mathcal{B}_{\delta} (x) \cap F|$), it follows by the
    ``sandwich'' criterion that both $\lim_{\delta \searrow 0}
    |\mathcal{B}_{\delta} (x) \cap E| / |\mathcal{B}_{\delta} (x) |$ and
    $\lim_{\delta \searrow 0} |\mathcal{B}_{\delta} (x) \cap F| /
    |\mathcal{B}_{\delta} (x) |$ exist and are equal to~$1$, which shows $x
    \in E^1 \cap F^1$.
    
    $'' \supseteq''$: Assume that $x \in E^1 \cap F^1$. Then
    \begin{eqnarray}
      1 & \geqslant & \lim_{\delta \searrow 0} \sup
      \frac{|\mathcal{B}_{\delta} (x) \cap E \cap F|}{|\mathcal{B}_{\delta}
      (x) |} \\
      & \geqslant & \lim_{\delta \searrow 0} \inf \frac{|\mathcal{B}_{\delta}
      (x) \cap E \cap F|}{|\mathcal{B}_{\delta} (x) |} \\
      & = & \lim_{\delta \searrow 0} \inf \frac{|\mathcal{B}_{\delta} (x)
      \cap E| + |\mathcal{B}_{\delta} \cap F| - |\mathcal{B}_{\delta} \cap (E
      \cup F) |}{|\mathcal{B}_{\delta} (x) |}. \nn
    \end{eqnarray}
    We obtain equality,
    \begin{eqnarray}
      1 & \geqslant & \lim_{\delta \searrow 0} \inf \frac{|\mathcal{B}_{\delta}
      (x) \cap E \cap F|}{|\mathcal{B}_{\delta} (x) |}\\
      & \geqslant & \lim_{\delta \searrow 0} \inf \frac{|\mathcal{B}_{\delta}
      (x) \cap E|}{|\mathcal{B}_{\delta} (x) |} + \lim_{\delta \searrow 0}
      \inf \frac{|\mathcal{B}_{\delta} (x) \cap F|}{|\mathcal{B}_{\delta} (x)
      |} +\cont\\
      & & \qq \lim_{\delta \searrow 0} \inf \left( - \frac{|\mathcal{B}_{\delta}
      \cap (E \cup F) |}{|\mathcal{B}_{\delta} (x) |} \right) \\
      & = & 2 - \underbrace{\lim_{\delta \searrow 0} \sup
      \frac{|\mathcal{B}_{\delta} \cap (E \cup F) |}{|\mathcal{B}_{\delta} (x)
      |}}_{\leqslant 1} \geqslant 1, 
    \end{eqnarray}
    from which we conclude that
    \begin{equation*}
      \lim_{\delta \searrow 0} \sup \frac{|\mathcal{B}_{\delta} (x) \cap E
      \cap F|}{|\mathcal{B}_{\delta} (x) |} = \lim_{\delta \searrow 0} \inf
      \frac{|\mathcal{B}_{\delta} (x) \cap E \cap F|}{|\mathcal{B}_{\delta}
      (x) |} = 1,
    \end{equation*}
    i.e.,~$x \in \left( E \cap F \right)^1$.
\end{proof}

\begin{proof}[Prop.~\ref{prop:bvcomposed}]
  First note that
  \begin{eqnarray}
    &  & \int_{\mathcal{F}E \cap \Omega} \|w^+_{\mathcal{F}E} -
    w^-_{\mathcal{F}E} \|_2 d\mathcal{H}^{d - 1} \\
    & \leqslant & \sup \{\|w^+_{\mathcal{F}E} (x) - w^-_{\mathcal{F}E}
    (x)\|_2 | x \in \mathcal{F}E \cap \Omega\} \mcont\\
    & & \qq\mathcal{H}^{d - 1}
    (\mathcal{F}E \cap \Omega) \\
    & \overset{(\ast)}{\leqslant} & \sup \{\|w (x) - w (y)\|_2 | x, y \in
    \Omega\} \cdot \tmop{TV} (1_E) \\
    & \overset{(\ast\ast)}{\leqslant} & \sqrt{2} \tmop{TV}
    (1_E) \\
    & = & \sqrt{2} \tmop{Per} (E) < \infty .  \label{eq:gs2per}
  \end{eqnarray}
  The inequality $(\ast)$ is a consequence of the definition of
  $w^{\pm}_{\mathcal{F}E}$ and {\cite[Thm.~3.59]{Ambrosio2000}}, and $(\ast\ast)$ follows directly from $w(x),w(y)\in\Delta_l$.
  The upper bound \eqref{eq:gs2per} permits applying {\cite[Thm.~3.84]{Ambrosio2000}} on $w$, which provides $w
  \in \tmop{BV} (\Omega)^l$ and (\ref{eq:dwedue}). Due to
  {\cite[Prop.~3.61]{Ambrosio2000}}, the sets $(E)^0,
  (E)^1$ and $\mathcal{F}E$ form a (pairwise disjoint) partition of $\Omega$,
  up to an $\mathcal{H}^{d - 1}$-zero set. Moreover, since $\Psi (D u) \ll |D
  u| \ll \mathcal{H}^{d - 1}$ by construction, we have, for some Borel set
  $A$,
  \begin{eqnarray}
    & & \int_A \Psi (D w)\\
    & = & \int_{A \cap (E)^1} d \Psi (D w) + \int_{A \cap
    (E)^0} d \Psi (D w) +\\
    &  & \int_{A \cap \mathcal{F}E \cap \Omega} \Psi \left( \nu_E  \left(
    w^+_{\mathcal{F}E} (x) - w^-_{\mathcal{F}E} (x) \right) ^{\top} \right)
    d\mathcal{H}^{d - 1}\nn\\
    & \overset{(\ast \ast)}{\leqslant} & \int_{A \cap (E)^1} d \Psi (D w) +
    \int_{A \cap (E)^0} d \Psi (D w) + \nn\\
    &  & \int_{A \cap \mathcal{F}E \cap \Omega} \lambda_u d\mathcal{H}^{d -
    1} \\
    & \overset{( \ref{eq:gs2per})}{\leqslant} & \int_{A \cap (E)^1} d \Psi (D
    w) + \int_{A \cap (E)^0} d \Psi (D w) + \nn\\
    & & \lambda_u \tmop{Per} (E) .
    \label{eq:intacape1}
  \end{eqnarray}
  The inequality $(\ast \ast)$ holds due to the upper boundedness and
  Prop.~\ref{prop:psimorebounds}. From {\cite[Prop. 2.37]{Ambrosio2000}} we
  obtain that $\Psi$ is additive on mutually singular Radon measures $\mu,
  \nu$, i.e., if $| \mu | \bot | \nu |$, then
  \begin{eqnarray}
    \int_B d \Psi (\mu + \nu) = \int_B d \Psi (\mu) + \int_B d \Psi (\nu) \hspace{1em}\label{eq:munuadd}
  \end{eqnarray}
  for any Borel set $B \subseteq \Omega$.
  Substituting $D w$ in (\ref{eq:intacape1}) according to (\ref{eq:dwedue})
  and using the fact that the three measures that form $D w$ in
  (\ref{eq:dwedue}) are mutually singular, the additivity property
  (\ref{eq:munuadd}) leads to the remaining assertion,
  \begin{eqnarray}
    &&\int_A d \Psi (D w) \leqslant \\
    &&\quad\int_{A \cap (E)^1} d \Psi (D u) +
    \int_{A \cap (E)^0} d \Psi (D v) + \lambda_u \tmop{Per} (E)\,.\nn
  \end{eqnarray}
\end{proof}

\begin{proof}[Prop.~\ref{prop:interiorsameobjective}]
  We first show (\ref{eq:ureeqvre}). It suffices to show that
  \begin{eqnarray}
    \left\{ x \in (E)^1 \right. & \Leftrightarrow & \left. x \in E \right\}
    \hspace{1em} \text{\tmop{for~}} \mathcal{L}^d \text{-a.e.~} x \in \Omega . 
    \label{eq:e1eqe}
  \end{eqnarray}
  This can be seen by considering the precise representative $\widetilde{1_E}$
  of $1_E$ {\cite[Def.~3.63]{Ambrosio2000}}: Starting with the definition,
  \begin{eqnarray}
    x \in (E)^1 & \Leftrightarrow & \lim_{\delta \searrow 0} \frac{|E \cap
    \mathcal{B}_{\delta} (x) |}{|\mathcal{B}_{\delta} (x) |} = 1\,,
    \end{eqnarray}
    the fact that  $\lim_{\delta \searrow 0} \frac{| \Omega \cap \mathcal{B}_{\delta} (x) |}{|\mathcal{B}_{\delta} (x) |} = 1$ implies
    \begin{eqnarray}
     x \in (E)^1 & \Leftrightarrow & \lim_{\delta \searrow 0} \frac{| (\Omega \setminus
    E) \cap \mathcal{B}_{\delta} (x) |}{|\mathcal{B}_{\delta} (x) |} = 0 \\
    & \Leftrightarrow & \lim_{\delta \searrow 0}
    \frac{1}{|\mathcal{B}_{\delta} (x) |} \int_{\mathcal{B}_{\delta} (x)} |1_E
    - 1| d y = 0 \\
    & \Leftrightarrow & \widetilde{1_E} (x) = 1. 
  \end{eqnarray}
  Substituting $E$ by $\Omega \setminus E$, the same equivalence shows that $x
  \in (E)^0 \Leftrightarrow \widetilde{1_{\Omega \setminus E}} (x) = 1
  \Leftrightarrow \widetilde{1_E} (x) = 0$. As $\mathcal{L}^d (\Omega
  \setminus ((E)^0 \cup (E)^1)) = 0$, this shows that $1_{E^1} =
  \widetilde{1_E}$ $\mathcal{L}^d$-a.e. Using the fact that $\widetilde{1_E} =
  1_E$ {\cite[Prop. 3.64]{Ambrosio2000}}, we conclude that $1_{(E)^1} = 1_E$
  $\mathcal{L}^d$-a.e., which proves (\ref{eq:e1eqe}) and therefore the
  assertion (\ref{eq:ureeqvre}).
  
  Since the measure-theoretic interior $(E)^1$ is defined over
  $\mathcal{L}^d$-integrals, it is invariant under $\mathcal{L}^d$-negligible
  modifications of $E$. Together with (\ref{eq:e1eqe}) this implies
  \begin{eqnarray}
    ((E)^1)^1 = (E)^1, \; \mathcal{F} (E)^1 =\mathcal{F}E,\;((E)^1)^0 = (E)^0\,.\label{eq:e11eqe1}
  \end{eqnarray}
  To show the relation $(D u) \llcorner (E)^1 = (D v) \llcorner (E)^1$,
  consider
  \begin{eqnarray}
    D u \llcorner (E)^1 & = & D \left( 1_{\Omega \setminus (E)^1} u +
    1_{(E)^1} u \right) \llcorner (E)^1 \\
    & \overset{(\ast)}{=} & D \left( 1_{\Omega \setminus (E)^1} u + 1_{(E)^1}
    v \right) \llcorner (E)^1 .  \label{eq:doneomsubse1}
  \end{eqnarray}
  The equality $(\ast)$ holds due to the assumption (\ref{eq:assue1}), and due
  to the fact that $D f = D g$ if $f = g$ $\mathcal{L}^d$-a.e. (see,
  e.g.,~{\cite[Prop. 3.2]{Ambrosio2000}}). We continue from
  (\ref{eq:doneomsubse1}) via
  \begin{eqnarray}
    & &  D u \llcorner (E)^1 \\
    & \overset{\tmop{Prop}. \ref{prop:bvcomposed}}{=} &
    \{D u \llcorner ((E)^1)^0 + D v \llcorner ((E)^1)^1 + \\
    &  & \nu_{(E)^1}  \left( u^+_{\mathcal{F}E^1} - v^-_{\mathcal{F}E^1}
    \right)^{\top} \mathcal{H}^{d - 1}_{} \llcorner (\mathcal{F}(E)^1 \cap
    \Omega) \} \llcorner (E)^1 \nonumber\\
    & \overset{( \ref{eq:e11eqe1})}{=} & \left( D u \llcorner (E)^0 + D v
    \llcorner (E)^1 \right) \llcorner \left( E \right)^1 + \\
    &  & \left( \nu_{(E)^1} \left( u^+_{\mathcal{F}E^1} -
    v^-_{\mathcal{F}E^1} \right)^{\top} \mathcal{H}^{d - 1}_{} \llcorner
    (\mathcal{F}E \cap \Omega) \right) \llcorner (E)^1 \nonumber\\
    & = & D u \llcorner \left( (E)^0 \cap (E)^1 \right) + D v \llcorner
    \left( (E)^1 \cap (E)^1 \right) + \\
    &  & \nu_{(E)^1} \left( u^+_{\mathcal{F}E^1} - v^-_{\mathcal{F}E^1}
    \right)^{\top} \mathcal{H}^{d - 1}_{} \llcorner (\mathcal{F}E \cap \Omega
    \cap (E)^1) \nonumber\\
    & = & D v \llcorner (E)^1 .  \label{eq:dudveq}
  \end{eqnarray}
  Therefore $D u \llcorner (E)^1 = D v \llcorner (E)^1$. Then,
  \begin{eqnarray}
    \Psi (D u) \llcorner (E)^1 & = & \Psi (D u \llcorner (E)^1 + \cont\\
    &&\quad D u \llcorner
    (\Omega \setminus (E)^1)) \llcorner (E)^1 \\
    & \overset{(\ast)}{=} & \Psi \left( D u \llcorner (E)^1 \right) \llcorner
    (E)^1 +\cont\\
    &&\quad \Psi \left( D u \llcorner (\Omega \setminus (E)^1) \right)
    \llcorner (E)^1 .  \label{eq:inte1psi}
  \end{eqnarray}
  In the equality $(\ast)$ we used the additivity of $\Psi$ on mutually
  singular Radon measures {\cite[Prop.~2.37]{Ambrosio2000}}. By definition of
  the total variation, $| \mu \llcorner A| = | \mu | \llcorner A$ holds for
  any measure~$\mu$, therefore $|D u \llcorner (\Omega \setminus (E)^1) | = |D
  u| \llcorner (\Omega \setminus (E)^1)$ and $|D u \llcorner (\Omega \setminus
  (E)^1) | ((E)^1) = 0$, which together with (again by definition) $\Psi (\mu)
  \ll | \mu |$ implies that the second term in~(\ref{eq:inte1psi}) vanishes.
  Since all observations equally hold for $v$ instead of $u$, we conclude
  \begin{eqnarray}
    \Psi (D u) \llcorner (E)^1 & = & \Psi (D u \llcorner (E)^1) \llcorner (E)^1\\
    & \overset{( \ref{eq:dudveq})}{=} & \Psi (D v \llcorner (E)^1) \llcorner (E)^1\\
    & = & \Psi (D v) \llcorner (E)^1 .
  \end{eqnarray}
  Equation (\ref{eq:ie1psi}) follows immediately.
\end{proof}


\bibliographystyle{spmpsci}      
\bibliography{references}

\end{document}